\documentclass[10pt,final,journal]{IEEEtran}

\usepackage{algorithm,algorithmicx,algpseudocode} 
\usepackage{amsmath,amsfonts,amssymb,amsthm} 
\usepackage{pifont,stmaryrd} 
\usepackage{graphicx,epstopdf,caption,subcaption} 
\usepackage{threeparttable,multirow} 
\usepackage[colorlinks]{hyperref} 
\usepackage[normalem]{ulem} 
\usepackage[usenames]{xcolor} 
\usepackage{cite} 
\usepackage{blindtext} 
\usepackage{balance}
\IEEEoverridecommandlockouts

\newcommand{\ba}        {\mathbf{a}}
\newcommand{\bA}        {\mathbf{A}}

\newcommand{\bC}        {\mathbf{C}}

\newcommand{\E}         {\mathbb{E}}

\newcommand{\cE}        {\mathcal{E}}

\newcommand{\bff}       {\mathbf{f}}

\newcommand{\bg}        {\mathbf{g}}

\newcommand{\cG}        {\mathcal{G}}

\newcommand{\bh}        {\mathbf{h}}

\newcommand{\bI}        {\mathbf{I}}

\newcommand{\cN}        {\mathcal{N}}

\newcommand{\cO}        {\mathcal{O}}

\newcommand{\bP}        {\mathbf{P}}

\newcommand{\bq}        {\mathbf{q}}

\newcommand{\R}         {\mathbb{R}}

\newcommand{\bs}        {\mathbf{s}}
\newcommand{\bS}        {\mathbf{S}}

\newcommand{\bv}        {\mathbf{v}}

\newcommand{\cV}        {\mathcal{V}}

\newcommand{\bW}        {\mathbf{W}}

\newcommand{\bx}        {\mathbf{x}}
\newcommand{\bX}        {\mathbf{X}}

\newcommand{\by}        {\mathbf{y}}
\newcommand{\bY}        {\mathbf{Y}}

\newcommand{\bSigma}    {\boldsymbol{\Sigma}}

\newcommand{\bone}      {\mathbf{1}}

\DeclareMathOperator*{\argmin}  {arg\,min}

\newcommand{\tT}        {\mathrm{T}}

\theoremstyle{plain}
\newtheorem{lemma}{Lemma}
\newtheorem{theorem}{Theorem}

\theoremstyle{definition}

\theoremstyle{remark}

\makeatletter
\newcommand*{\rom}[1]{\expandafter\@slowromancap\romannumeral #1@}
\makeatother

\graphicspath{{./}} 

\begin{document}

\title{FAST-PCA: A Fast and Exact Algorithm for Distributed Principal Component Analysis}

\author{Arpita Gang and Waheed U.\ Bajwa
\thanks{Preliminary versions of some of the results reported in this paper were presented at the 2019 IEEE International Conference on Acoustics, Speech and Signal Processing (ICASSP), Brighton, United Kingdom, May 12-17 2019~\cite{gang.raja.bajwa.2019}. AG and WUB are with the Department of Electrical and Computer Engineering, Rutgers University--New Brunswick, NJ 08854 (Emails: {\tt \{arpita.gang,~waheed.bajwa\}@rutgers.edu}).}

\thanks{This work was supported in part by the National Science Foundation under Awards CCF-1907658, and OAC-1940074, and by the Army Research Office under Awards W911NF-17-1-0546 and W911NF-21-1-0301.}}

\maketitle
	
\begin{abstract}
Principal Component Analysis (PCA) is a fundamental data preprocessing tool in the world of machine learning. While PCA is often thought of as a dimensionality reduction method, the purpose of PCA is actually two-fold: dimension reduction and uncorrelated feature learning. Furthermore, the enormity of the dimensions and sample size in the modern day datasets have rendered the centralized PCA solutions unusable. In that vein, this paper reconsiders the problem of PCA when data samples are distributed across nodes in an arbitrarily connected network. While a few solutions for distributed PCA exist, those either overlook the uncorrelated feature learning aspect of the PCA, tend to have high communication overhead that makes them inefficient and/or lack `exact' or `global' convergence guarantees. To overcome these aforementioned issues, this paper proposes a distributed PCA algorithm termed \textit{FAST-PCA (Fast and exAct diSTributed PCA)}. The proposed algorithm is efficient in terms of communication and is proven to converge linearly and exactly to the principal components, leading to dimension reduction as well as uncorrelated features. The claims are further supported by experimental results.
\end{abstract}

\section{Introduction}\label{sec:intro}
Massive and high-dimensional datasets are becoming an increasingly essential part of the modern world ranging from healthcare to finance and from social media to the Internet-of-Things (IoT). In a related trend, machine learning algorithms are finding their applications in every possible domain because of their data-driven nature and the ability to generalize to new unseen data. But these algorithms need a considerable amount of data preprocessing for their effective and efficient use. One of the major steps in this preprocessing is dimension reduction and feature learning for compression and extraction of useful features from raw data that can be used in downstream machine learning algorithms for classification, clustering, etc. Principal Component Analysis (PCA)~\cite{Hotelling.1933} is a workhorse tool for such dimension reduction and feature extraction purposes. In a nutshell, PCA transforms a large set of correlated features to a smaller set of uncorrelated features that contain maximum information of the raw data. 

The increasing volume of available data along with concerns like privacy, communication cost, etc., as well as emerging applications such as smart cities, autonomous vehicles, etc., have also led to a significant interest in the development of distributed algorithms for doing PCA on non-collocated data in the last couple of decades~\cite{BajwaCevherEtAl.ISPM20}. Data tends to be distributed for a multitude of reasons; it can be inherently distributed like in IoT, sensor networks, etc., or it can be distributed due to storage and/or computational limitations. The ultimate goal of any distributed algorithm is to solve a common problem using data shared among the distributed entities through communication with each other so that all entities collectively reach a solution that is nearly as good as the solution of the centralized algorithms, for which data is available at a single location. Motivated by these reasons, we develop and analyze an effective solution for distributed PCA that is efficient in terms of communication, that does not require exchange of raw data and that can be proved to converge exactly for any arbitrary network topology and at a linear rate to a solution that is the same as the one returned by \textit{centralized} PCA.

Distributed setups can be largely classified into two types: i) those having a central entity/server that coordinates with the other nodes in a master-slave architecture, and ii) those lacking any central entity, in which the nodes are connected in an arbitrary network. In the first type of setup, the central entity aggregates information from all the nodes and yields the final result. Since the second type of architecture does not rely on any central entity, it is a more general setup and it lacks a single point of failure. The detailed review in~\cite{Yang.Gang.Bajwa.2020} discusses these setups along with various algorithms developed for both in more detail. Although the terms distributed and decentralized are used interchangeably for both setups in the literature, we consider the latter scenario for the distributed PCA and call it \textit{distributed} in this paper.

The goal of dimension reduction can be accomplished by learning a low-dimensional subspace spanned by the dominant eigenvectors of the covariance matrix of the distribution to which the data samples belong. Mathematically speaking, for a data point $\by \in \R^d$ sampled from a distribution with zero mean and covariance $\bSigma \in \R^{d\times d}$, dimension reduction can be achieved by projecting $\by$ onto a matrix $\bX \in \R^{d\times K}, K \ll d$, such that $\bX$ spans a subspace spanned by the leading $K$ eigenvectors of $\bSigma$ under the constraint $\bX^T\bX = \bI$, that is $\bX$ lies on a Stiefel manifold. When $\by$ is compressed as $\tilde{\by} = \bX^\tT\by$ with such an $\bX$, its reconstruction $\bX\bX^\tT\by$ has minimum error in the Frobenius norm sense. However, this approach can only be called \textit{principal subspace analysis} as it does not ensure that the resultant $K$ features in $\tilde{\by}$ are uncorrelated. The uncorrelatedness constraint requires $\E\begin{bmatrix}\tilde{\by}\tilde{\by}^T\end{bmatrix}= \E\begin{bmatrix}\bX^\tT\by\by^\tT\bX\end{bmatrix}$ to be a diagonal matrix, which is fulfilled only when $\bX$ contains the eigenvectors of $\bSigma$, not just any orthogonal basis of the subspace spanned by the said eigenvectors. The true purpose of PCA is thus fulfilled by a specific element of the Stiefel manifold that corresponds to the eigenvectors of $\bSigma$. 

An autoencoder is another popular neural-network based tool for data compression. The good generalization capability of neural network-based systems along with their ease of parallelization in the case of massive data make them very attractive and efficient solutions for PCA. A study in~\cite{baldi.hornik.1989} showed that the optimum weights of an autoencoder for efficient data compression and features decorrelation, when the loss function is the reconstruction error, are given by the space spanned by the eigenvectors of the input covariance matrix. It was also noted in~\cite{oja, sanger} that neural networks trained using the Hebbian learning rule~\cite{Hebb.1949} extract principal components of the input correlation matrix in the streaming data case. In the same vein of streaming setting, an earlier work by Krasulina~\cite{krasulina} proposed a similar learning method that converges to the dominant eigenvector of the expectation of input sample covariance matrix. Even though the methods proposed by Krasulina~\cite{krasulina} and Oja~\cite{oja} have many similarities as pointed out in~\cite{oja, balsubramani2013}, Oja's~\cite{oja} rule has been studied more extensively than Krasulina's~\cite{krasulina} method. The original Krasulina's method is a simple iterative method for the estimation of the top eigenvector in the streaming case and its matrix version called Matrix Krasulina was proposed much later in~\cite{matrix_krasulina2019} that extends the original method to estimate the subspace spanned by the top $K$ eigenvectors. Since we aim to find the top $K$ eigenvectors, in this paper we propose a learning method based on the original Krasulina's method that can be shown to converge to the first $K$ eigenvectors (principal components) of the sample covariance matrix, not just the principal subspace, in distributed batch settings. Due to the parallelization potential and an iterative update-based rule, our proposed method is applicable to autoencoder training as well.
\subsection{Relation to Prior Work}
The problem of dimension reduction goes back to as early as 1901 when Pearson~\cite{Pearson.1901} aimed at fitting a line to a set of data points. Later, Hotelling~\cite{Hotelling.1933} proposed a PCA method for decorrelating and compressing a set of data points by finding their principal components. Since then, many iterative methods like power method, orthogonal iterations~\cite{Golub}, Lanczos method~\cite{Lanczos.1950} have been proposed to estimate eigenvectors or low-dimensional subspaces of symmetric matrices, a class under which covariance matrices fall. A stochastic approximation algorithm was proposed by Krasulina in~\cite{krasulina} for the estimation of the dominant eigenvector in the streaming data case. From the point of view of training neural networks for data compression, an algorithm very similar to Krasulina's method was later proposed by Oja~\cite{oja}, which was then extended for multiple eigenvector estimation by Sanger~\cite{sanger}. Both Oja's and Sanger's method were based on the Hebbian learning rule~\cite{Hebb.1949} and it was shown that the weights of an autoencoder trained using this rule converge to the eigenvectors of the input correlation matrix. The works in~\cite{yi.tan2005, lv.yi.tan.2007} proved that in deterministic batch settings Oja's rule and the generalized Hebbian rule proposed by Sanger converge to the eigenvectors of a covariance matrix at a linear rate. Krasulina's method was also generalized for the estimation of a subspace of dimension greater than one in~\cite{matrix_krasulina2019}, although it only guarantees convergence to the principal subspace, instead of principal components, at a linear rate under the low-rank matrix assumption. 

The problem of PCA in the distributed setting is relatively recent. In any distributed network, data can be distributed by either features or samples and the solutions for these two data distribution types are significantly different. A detailed review of various distributed PCA algorithms for both kinds of data distribution is done in~\cite{wu2018review}. For the case of feature-wise distribution as in~\cite{mcsherry, scaglione.krim.2008, d-oja}, each node in the network estimates one or a subset of features of the entire subspace. In this paper we focus on the case of sample-wise data distribution, where each node estimates the entire basis and consensus in the network is a necessary condition. The sample-wise data distribution was considered in~\cite{cksvd.allerton.2013, cksvd, depm}, where a power method-based approach was proposed for estimation of the dominant eigenvector ($K=1$). This method requires an explicit consensus loop~\cite{consensus} in every iteration of the power method and the final error is a function of the number of consensus iterations. The power method-based distributed PCA solutions can be used for multiple ($K>1$) eigenvector estimation in a sequential manner, where lower-order eigenvectors are estimated using the residue of the covariance matrix left after its projection on the higher-order eigenvectors. Since estimation of any lower-order eigenvector requires that the higher-order eigenvectors are fully estimated, this sequential approach results in a rather slow algorithm. To overcome the issues of the sequential approach, an orthogonal iteration-based solution for the case of $K>1$ was proposed in~\cite{xiang.gang.bajwa.2021}. Although this method estimates the $K$-dimensional subspace simultaneously, its convergence guarantees are in terms of subspace angles and thus it proves convergence to the principal subspace. Moreover, all these aforementioned methods require an explicit consensus loop making these algorithms inefficient in terms of communication overhead. 

PCA is a non-convex problem since the uncorrelated constraint requires the solution to be a specific element on the Stiefel manifold. Recently, some algorithms in the field of distributed optimization were proposed to deal with non-convex problems. While some of those deal with unconstrained problems~\cite{proxpda}, some are developed for non-convex objectives with convex constraints~\cite{bianchi.jakubowicz.2013, wai.scaglione.lafond.2016}, while some methods guarantee convergence only to a stationary point~\cite{chen.mingyi.2021}. For these reasons, none of the existing distributed algorithms for non-convex problems are directly applicable for the PCA objective. A recent work based on perturbation theory for linear operators based on the Picard iteration was proposed for distributed optimization in~\cite{picard}. The extension of this work in~\cite{picardPCA} demonstrated the application of the distributed Picard iteration (DPI) method to distributed PCA but it could only prove local convergence, i.e., if the estimate is already ``close enough" to the optimal solution, then it converges to the optimal point at a linear rate. Furthermore, the DPI method suffers from two more limitations in terms of its theoretical analysis, namely it requires the covariance matrix to be full rank as well as the upper bound on step size required for convergence guarantees is not quantified in terms of problem parameters like eigengap, data dimension etc. Thus, many gaps still remain to be filled in distributed PCA.

The work in this paper is an extension of our preliminary work in~\cite{gang.raja.bajwa.2019} that proposed two fast and efficient algorithms for distributed PCA in the case of sample-wise distributed data. A distributed algorithm for PCA based on generalized Hebbian algorithm using a combine-and-adapt strategy called distributed Sanger's algorithm (DSA) was developed and analyzed in our previous work~\cite{gang.bajwa.2021}. Although this strategy has mainly been used in distributed optimization literature for convex and strongly convex problems, we showed using extensive analysis that even for the non-convex PCA problem, each node converges linearly and globally, i.e., starting from any random initial point. Though it is a linearly convergent one-time scale algorithm, it only reaches to a neighborhood of the optimal solution for a fixed step size. The algorithm, however, does converge exactly in the case of decreasing step sizes but with a slower rate of convergence. This result is coherent with the combine-and-update based gradient descent solutions~\cite{dgd} for distributed optimization. To overcome such limitations of simple gradient descent-based algorithms, some new methods have been proposed recently that deploy a technique called ``gradient-tracking", which has been shown to converge exactly in the case of convex~\cite{extra,qu.li.2018}, strongly convex and some non-convex problems~\cite{next}. In this paper, we use this gradient-tracking idea to develop an algorithm for the \textit{non-convex} distributed PCA problem that linearly converges to an optimal solution that is the same as its centralized counterpart. 

A very recent paper on distributed PCA~\cite{deepca} used this gradient-tracking idea to develop a two-time scale algorithm called DeEPCA for subspace estimation. Our work has three major differences as compared to DeEPCA: firstly, our algorithm guarantees convergence to the eigenvectors of the global covariance matrix and not just any rotated basis of the same subspace, thereby making our algorithm a true PCA and not just a principal subspace analysis (PSA) solution. Secondly, we do not use any explicit consensus loop for ensuring agreement in the network, making it a very communication-efficient solution and finally, DeEPCA requires explicit QR decomposition in every iteration unlike our algorithm, thus requiring more computations.

\subsection{Our Contributions}
The main contributions of this paper are 1) a novel algorithm for distributed PCA called \textit{Fast and exAct diSTributed PCA} (FAST-PCA) based on a generalization of Krasulina's method, 2) theoretical guarantees that show that the estimates given by our method converge exactly and globally at a linear rate to the eigenvectors of the global covariance matrix, and 3) experimental results that further demonstrate the efficiency of our solution for both synthetic and real-world datasets.

Our primary focus in this paper is to develop a solution for distributed PCA when the data samples are scattered across an arbitrarily connected network with no central node. While PCA is often reduced to dimension reduction, we focus on the dual goal of PCA that requires dimensionality reduction as well as feature decorrelation. To that end, we propose an algorithm based on Krasulina's method using a gradient-tracking approach. Since the original Krasulina's method only finds the dominant eigenvector, we also generalize it to the distributed setting for the estimation of top $K$ eigenvectors.  Our proposed FAST-PCA method is an iterative update algorithm and its main attributes are that it is fast since it lacks any explicit consensus loop and hence reduces the communication overhead, and it converges exactly to the true eigenvectors of the global covariance matrix at a linear rate. We provide detailed convergence analysis to support our claims as well as extensive numerical experiments where we compare our method to centralized orthogonal iteration (OI) as the centralized baseline, as well as distributed PCA algorithms of sequential distributed power method (SeqDistPM), DeEPCA and DSA. We provide the results for different network topologies as well as eigengaps to further solidify our claims.

To the best of our knowledge, this is a first novel algorithm for distributed PCA based on Krasulina's method which achieves fast and exact convergence to the true eigenvectors of the global covariance matrix at every node of an arbitrarily connected network.

\subsection{Notation and Organization}
The following notation is used in this paper. Scalars and vectors are denoted by lower-case and lower-case bold letters, respectively, while matrices are denoted by upper-case bold letters. The operator $|\cdot|$ denotes the absolute value of a scalar quantity. The superscript in $\ba^{(t)}$ denotes time (or iteration) index, while $a^t$ denotes the exponentiation operation. The superscript $(\cdot)^\tT$ denotes the transpose operation, the operator $\otimes$ denotes Kronecker product, $\|\cdot\|_F$ denotes the Frobenius norm of matrices, while both $\|\cdot\|$ and $\|\cdot\|_2$ denote the $\ell_2$-norm of vectors. Given a matrix $\bA$, both $a_{ij}$ and $(\bA)_{ij}$ denote its entry at the $i^{th}$ row and $j^{th}$ column, while $\ba_j$ denotes its $j^{th}$ column. The matrix $\bI_a \in \R^{a\times a}$ denotes the identity matrix of dimension $a$.

The rest of the paper is organized as follows. In Section~\ref{sec:problem}, we describe and mathematically formulate the distributed PCA problem, while Section~\ref{sec:algo} describes the proposed distributed algorithm, which is based on Krasulina's algorithm. In Section~\ref{sec:analysis_centralized}, we derive an auxiliary result based on Krasulina's method that aids in the convergence analysis of the proposed distributed algorithm, while convergence guarantees for the proposed algorithm are provided in Section~\ref{sec:analysis_fastpca}. We provide numerical results in Section~\ref{sec:results} to show efficacy of the proposed method and provide concluding remarks in Section~\ref{sec:conc}.

\section{Problem Description}\label{sec:problem}
Principal Component Analysis (PCA) is a widely used data preprocessing tool to find a low-dimensional subspace that would decorrelate data features while retaining maximum information. For data samples $\by \in \R^d$ sampled from a zero-mean distribution with covariance matrix $\bSigma$, PCA can be mathematically formulated as
\begin{align}\nonumber
    \bX &= \underset{\bX \in \R^{d\times K}, \bX^{\tT}\bX = \bI} \argmin \quad \E\begin{bmatrix}\|\by-\bX\bX^{\tT}\by\|_2^2\end{bmatrix} \\\label{eq:pca}
    &\text{such that}\qquad \forall l\neq q, \ \Big(\E \begin{bmatrix} \bX^{\tT}\by\by^{\tT}\bX \end{bmatrix}\Big)_{lq} = 0.
\end{align}
The constraint $\Big(\E \begin{bmatrix} \bX^{\tT}\by\by^{\tT}\bX \end{bmatrix}\Big)_{lq} = 0, \forall l\neq q$, ensures that $\bX$ decorrelates the features of $\by$. It is evident that $\E \begin{bmatrix} \bX^{\tT}\by\by^{\tT}\bX \end{bmatrix} = \bX^{\tT}\E \begin{bmatrix} \by\by^{\tT} \end{bmatrix}\bX$ will be a diagonal matrix if and only if $\bX$ contains the eigenvectors of $\E \begin{bmatrix} \by\by^{\tT} \end{bmatrix} = \bSigma$. Thus the search for a solution of PCA not only requires a minimum reconstruction error solution, which will be given by any basis of the subspace spanned by the dominant $K$ eigenvectors of the covariance matrix $\bSigma$, but the basis vectors should specifically be the eigenvectors of $\bSigma$. In practice the actual distribution of the samples and hence $\bSigma$ is unknown and a sample covariance matrix is used instead for PCA. For a set of samples $\{\by_t\}_{t=1}^N$, the sample covariance matrix is given by $\bC = \frac{1}{N-1}\sum_{t=1}^{N}(\by_t - \bar{\by})(\by_t - \bar{\by})^T$, where $\bar{\by} = \frac{1}{N}\sum_{t=1}^{N}\by_t$ is the sample mean. Henceforth, we shall assume $\bar{\by} = 0$ without loss of generality because the mean can otherwise be calculated and subtracted from the samples. The empirical formulation of the PCA problem in terms of samples is thus given as
\begin{align}\nonumber
    \bX  &= \underset{\bX \in \R^{d\times K}, \bX^{\tT}\bX = \bI} \argmin \quad \sum_{t=1}^N\|\by_t-\bX\bX^{\tT}\by_t\|_2^2 \\\label{eq:pca1}
    &\text{such that}\quad \forall l\neq q, \Big(\bX^{\tT}(\sum_{t=1}^{N}\by_t\by_t^{\tT})\bX\Big)_{lq} = 0.
\end{align}

A distributed setting implies that the entire data matrix $\bY = \begin{bmatrix}\by_1,\ldots,\by_N\end{bmatrix} \in \R^{d\times N}$ is unavailable at a single location. Let us consider an undirected and connected network of $M$ nodes described by a graph $\cG = \{\cV, \cE\}$, where $\cV = \{1,\ldots,M\}$ is the set of nodes and $\cE$ is the set of edges between the nodes. For each node $i$, the set of its directly connected neighbors is given by $\cN_i$. The data can be distributed among the nodes along the rows, i.e., by features or along the columns, i.e., by samples. In this paper, we consider the case when the samples $\{\by_t\}_{t=1}^N$ are scattered spatially over a network. Thus, each node $i \in \cV$ has a subset of the samples $\bY_i \in \R^{d\times N_i}$ such that $\bY = \begin{bmatrix}\bY_1,\ldots,\bY_M\end{bmatrix}$. The PCA formulation in this distributed case is: 
\begin{align}\nonumber
    \bX &= \underset{\bX \in \R^{d\times K}, \bX^{\tT}\bX = \bI} \argmin \quad \sum_{i=1}^M\|\bY_i-\bX\bX^{\tT}\bY_i\|_F^2 \\ \label{eq:pca_d}
    &\text{such that}\qquad \forall l\neq q, \Big(\bX^{\tT}(\sum_{i=1}^{M}\bY_i\bY_i^{\tT})\bX\Big)_{lq} = 0.
\end{align}

Although the formulations~\eqref{eq:pca1} and~\eqref{eq:pca_d} look similar, a major difference is the unavailability of $\bY_i$'s at a single location rendering the methods for solving~\eqref{eq:pca1} unusable directly for solving~\eqref{eq:pca_d}. Since each node carries different local data, there is a difference in local objective function even though the constraint is globally shared. This in turn leads to each node maintaining its own copy $\bX_i$ of the variable $\bX$ . As mentioned before, the goal of distributed PCA is for each node to eventually reach the same solution, i.e., achieve network consensus, given by the eigenvectors of $\bC$. Thus, the actual PCA objective for the distributed case is
\begin{align}\nonumber
&\underset{\bX_i \in \R^{d\times K}, \bX_i^{\tT}\bX_i = \bI} \argmin \quad \sum_{i=1}^M\|\bY_i-\bX_i\bX_i^{\tT}\bY_i\|_F^2 \quad \text{such that}\\ \label{eq:pca_d2}
&\forall j\in\cN_i, \bX_i = \bX_j \quad \text{and}\quad \forall l\neq q, \Big(\bX_i^{\tT}(\sum_{i=1}^{M}\bY_i\bY_i^{\tT})\bX_i\Big)_{lq} = 0.
\end{align}
Since each node $i$ has access to a subset of data points $\bY_i$ and subsequently has a local covariance matrix $\bC_i = \frac{1}{N_i}\bY_i\bY_i^\tT$, a naive solution is that each node solves its own PCA formulation as follows:
\begin{align}\nonumber
    \bX_i &= \underset{\bX_i \in \R^{d\times K}, \bX_i^{\tT}\bX_i = \bI} \argmin \quad \|\bY_i-\bX_i\bX_i^{\tT}\bY_i\|_F^2 \\ \label{eq:pca_d11}
    & \text{such that} \qquad \forall l\neq q, \Big(\bX_i^{\tT}\bY_i\bY_i^{\tT}\bX_i\Big)_{lq} = 0.
\end{align}
However, the naive solution of~\eqref{eq:pca_d11} will have major drawbacks. As explained earlier, PCA ideally aims to find the eigenvectors of covariance $\bSigma$ of the distribution the data points are sampled from but instead uses sample covariance matrix $\bC$ because $\bSigma$ is unknown in practice and $\E\begin{bmatrix}\bC\end{bmatrix} = \bSigma$. Since $\bC\rightarrow \bSigma$ as the number of samples $N$ increases, using only the local covariance matrices would incur a higher loss in the estimation of the eigenvectors. Furthermore, it is plausible that the samples at a single node are not uniformly sampled from the entire distribution and hence any estimation made using local covariances would result in a biased estimate. These reasons dictate that all the $N$ samples in the network should be incorporated somehow in the estimation of the eigenvectors for dimension reduction and decorrelation at all the nodes of the network. Additionally, in the case of sample-wise distributed data, all nodes should agree and converge to a common solution that is the same as the solution of~\eqref{eq:pca1} when all the samples are available at a single location. 

The constraint in~\eqref{eq:pca_d2} has two important properties. First, since the solution lies on the Stiefel manifold and particularly, it is a specific element of the manifold, the problem is non-convex. Although this issue can be dealt with through convex approximation of the problem~\cite{arora2013stochastic}, such an approach will result in $\cO(d^2)$ computational and memory requirements since it approximates the projection matrix of the $d\times K$ dimensional subspace and that can be restrictive in the case of high-dimensional data. At the same time, such convexification leads to a relaxed constraint that would only give a rotated basis of the subspace spanned by the eigenvectors of $\bC$ and not the eigenvectors themselves. Second, the constraint $\bX_i^{\tT}(\sum_{i=1}^M\bY_i\bY_i^{\tT})\bX_i$ being diagonal is shared by all nodes due to the reasons explained earlier. Thus meeting this global constraint requires that all nodes of the network collaborate to reach a common solution $\bX = \bX_i, \forall i \in \cV$. Hence, in this paper we propose an iterative algebraic method based on Krasulina's rule~\cite{krasulina} for distributed PCA that ensures that all nodes simultaneously converge to the eigenvectors of the global covariance matrix $\bC$ without having to share their local covariance $\bC_i$. The algorithm converges exactly to the eigenvectors of the global covariance matrix $\bC$ at a linear rate when the error is measured in terms of angles between the estimates and the true eigenvectors.

\section{Proposed Algorithm: FAST-PCA}\label{sec:algo}
Iterative solutions such as the power method, Oja's rule and Krasulina's method have proven to be powerful tools for PCA, i.e., dimension reduction and simultaneous feature decorrelation in centralized settings when the data is collocated or streaming at a single location. Although Krasulina's and Oja's method have similar update rules, in this paper we extend the Krasulina's method to develop an algorithm for distributed PCA in batch settings. The original Krasulina's method was developed as a stochastic approximation algorithm for estimating the dominant eigenvector of the expected  correlation matrix (which is same as covariance matrix for zero-mean inputs) in case of streaming data. Let $\by_t, t = 1,2,\ldots$ be data sample drawn from a zero-mean distribution at time $t$. Then Krasulina's method estimated the leading eigenvector of $\bSigma = \E\begin{bmatrix}\by_t\by_t^\tT\end{bmatrix}$ by the following update equation:
\begin{align}\label{eq:krasulina}
    \bx^{(t+1)} = \bx^{(t)} + \alpha_t\Big(\bC_t\bx^{(t)} - \frac{(\bx^{(t)})^T\bC_t\bx^{(t)}}{\|\bx^{(t)}\|^2}\bx^{(t)}\Big),
\end{align}
where $\bC_t = \by_t\by_t^\tT$ is the covariance matrix obtained from one sample and $\alpha_t$ is the step size at time $t$. It was proved in~\cite{krasulina} that if the spectral norm of $\E\begin{bmatrix}\by_t\by_t^\tT\end{bmatrix}$ remains bounded and $\sum_t \alpha_t^2$ converges to zero as $t\rightarrow \infty$, the update equation~\eqref{eq:krasulina} yields the dominant eigenvector of $\E\begin{bmatrix}\bC_t\end{bmatrix}$. 

In the distributed setup considered in this paper, samples are not streaming but distributed across a connected network of $M$ nodes, where node $i$ has access to a local covariance matrix $\bC_i$ such that $\sum_{i=1}^M\bC_i = \bC$, the global covariance matrix. It is noteworthy that $\E\begin{bmatrix}\bC_t\end{bmatrix}=\E\begin{bmatrix}\bC_i\end{bmatrix} = \bSigma$ and this similarity between streaming and distributed setting motivates the extrapolation of Krasulina's method for the distributed setting. For the dominant eigenvector $K=1$, a naive approach would be for each node to estimate an eigenvector using its local data and update rule~\eqref{eq:krasulina}. However, that would result in each node $i$ to only estimate the dominant eigenvector of $\bC_i$ whereas the goal of distributed PCA is for every node to estimate the eigenvector of the global covariance matrix $\bC$. Furthermore, since Matrix Krasulina~\cite{matrix_krasulina2019} only estimates the dominant subspace, Krasulina's method also needs to be generalized for the estimation of $K > 1$ dominant \textit{eigenvectors}. 

Estimation of the eigenvectors of $\bC$ at every node without sharing raw local covariance matrix $\bC_i$ would require some form of collaboration among the nodes of the network. As mentioned earlier, our previous work~\cite{gang.bajwa.2021} used a combine-and-adapt strategy in a way that each node converges linearly but only to a neighborhood of the true eigenvectors of the global covariance matrix $\bC$. Even though we used the generalized Hebbian algorithm~\cite{sanger}, some straightforward calculations and manipulations can show similar results for Krasulina's method. In this paper, we aim to fill that gap of inexact convergence and propose a gradient-tracking based solution~\cite{extra, qu.li.2018} that converges exactly and linearly to the true eigenvectors of $\bC$ at every node. If $\bx_{i,1}^{(t)}$ is the estimate of the dominant eigenvector at node $i$ after the $t^{th}$ iteration, then we define a pseudo-gradient at node $i$ as follows:
\begin{align}
    \bh_i(\bx_{i,1}^{(t)}) = \bC_i\bx_{i,1}^{(t)} - \frac{(\bx_{i,1}^{(t)})^T\bC_i\bx_{i,1}^{(t)}}{\|\bx_{i,1}^{(t)}\|^2}\bx_{i,1}^{(t)},
\end{align}
which is similar to the update portion of~\eqref{eq:krasulina}. Additionally, for the estimation of $k^{th}, k = 2,\ldots,K$, eigenvector we propose to generalize Krasulina's update rule along the lines of the generalized Hebbian algorithm~\cite{sanger} and combine Krasulina's method with Gram-Schmidt orthogonalization to define a general pseudo-gradient as:
\begin{align}
    \bh_i(\bx_{i,k}^{(t)}) = \bC_i\bx_{i,k}^{(t)} - \frac{(\bx_{i,k}^{(t)})^T\bC_i\bx_{i,k}^{(t)}}{\|\bx_{i,k}^{(t)}\|^2}\bx_{i,k}^{(t)} - \sum_{p=1}^{k-1}\frac{(\bx_{i,p}^{(t)})^T\bC_i\bx_{i,k}^{(t)}}{\|\bx_{i,p}^{(t)}\|^2}\bx_{i,p}^{(t)}.
\end{align}
Here, the term $\frac{(\bx_{i,p}^{(t)})^T\bC_i\bx_{i,k}^{(t)}}{\|\bx_{i,p}^{(t)}\|^2}\bx_{i,p}^{(t)}$ is analogous to Gram-Schmidt orthogonalization and enforces the orthogonality of $\bx_{i,k}^{(t)}$ to $\bx_{i,p}^{(t)}, p = 1,\ldots, k-1$. 

Let $\bX_i^{(t)} = \begin{bmatrix}\bx_{i,1}^{(t)},\ldots,\bx_{i,K}^{(t)}\end{bmatrix} \in \R^{d\times K}$ be the estimate of the $K$ eigenvectors of the global covariance matrix $\bC$. A gradient tracking-based algorithm also updates a second variable~\cite{next, qu.li.2018} in every iteration that essentially tracks the average of the gradients at the nodes. In a similar fashion, let us define a \textit{pseudo-gradient tracker} matrix $\bS_i^{(t)} = \begin{bmatrix}\bs_{i,1}^{(t)},\ldots,\bs_{i,K}^{(t)}\end{bmatrix} \in \R^{d\times K}$ that tracks the average of the pseudo-gradients at each node. These $\bS_i^{(t)}$ are updated along with the eigenvector estimates $\bX_i^{(t)}$ in each iteration of our algorithm \textit{Fast and exAct diSTributed PCA (FAST-PCA)}, which is described in Algorithm~\ref{algo:edsa}. At each node $i$, the eigenvector estimates $\bX_j^{(t)}, j \in \cN_i$, where $\cN_i$ is the set of neighbors of node $i$, are combined as a weighted average and updated with the local copy of the gradient tracker $\bS_i^{(t)}$ using a constant step size $\alpha$. Along with that, $\bS_i^{(t)}$ is also updated as a weighted average of $\bS_j^{(t)}$ and difference of pseudo-gradients. The entity $\bh_i(\bX_i^{(t)})$ in the algorithm is the matrix of the psuedo-gradients, i.e., $\bh_i(\bX_i^{(t)})=\begin{bmatrix}\bh_i(\bx_{i,1}^{(t)}),\ldots,\bh_i(\bx_{i,K}^{(t)})\end{bmatrix} \in \R^{d\times K}$. The weight matrix $\bW = \begin{bmatrix}w_{ij}\end{bmatrix}$ is a doubly stochastic matrix that conforms to the underlying graph topology~\cite{Boydfastestmixing.2003}, i.e., $w_{ij} \neq 0$ if $(i,j) \in \cE$ or $i=j$ and 0 otherwise. A necessary assumption for convergence of the algorithm here is the graph connectivity, which ensures that the magnitude of the second largest eigenvalue of $\bW$ is strictly less than 1. The gradient-tracking based solutions are recently being very popular in distributed optimization literature because of their fast and exact convergence guarantees. Our main challenge here was providing theoretical convergence guarantees inspite of the non-convex nature of the problem. In the next section, we provide detailed analysis of our proposed algorithm FAST-PCA and show that the estimates $\bx_{i,k}^{(t)}$ at each node $i$ converge at a linear rate $\cO(\rho^t), 0<\rho<1$, for any random unit-norm initialization and a certain condition on step size, to the eigenvectors $\pm \bq_k$ of the global covariance matrix $\bC$.
\begin{algorithm}[ht]
	\textbf{Input:} $\bY_1,\bY_2, \dots \bY_M, \bW, \alpha, K$\\
	\textbf{Initialize:} $\forall i, \bX_i^{(0)} \gets \bX_{\text{init}}: \bX_{\text{init}} \in \R^{d\times K}, \bX_{\text{init}}^T \bX_{\text{init}} = \bI$; $\bS_i^{(0)} \gets \bh_i(\bX_i^{(0)})$
	\begin{algorithmic}
		\For{$t=0,1,\dots$}
		    \State Communicate $\bX_i^{(t)}$ from each node $i$ to its neighbors
		    \State Subspace estimate at node $i$: $\mbox{\quad} \bX_i^{(t+1)} \gets \frac{1}{2}\bX_{i}^{(t)} + \sum_{j\in \cN_i}\frac{w_{ij}}{2}\bX_{j}^{(t)} + \alpha \bS_i^{(t)}$
		    \State Psuedo-gradient estimate at node $i$: $\mbox{\quad} \bS_i^{(t+1)} \gets \frac{1}{2}\bS_{i}^{(t)} + \sum_{j\in \cN_i}\frac{w_{ij}}{2}\bS_{j}^{(t)} + \bh_i(\bX_i^{(t+1)}) - \bh_i(\bX_i^{(t)})$
		\EndFor
	\end{algorithmic}
	{\bf Return:} $\tilde{\bX}_i^{(t+1)}= \begin{bmatrix}\frac{\bx_{i,1}^{(t+1)}}{\|\bx_{i,1}^{(t+1)}\|},\ldots,\frac{\bx_{i,K}^{(t+1)}}{\|\bx_{i,K}^{(t+1)}\|}\end{bmatrix}, i = 1,2, \dots, M$
	\caption{Fast and exAct diSTributed PCA (FAST-PCA)}
	\label{algo:edsa}
\end{algorithm}

\section{Convergence Analysis}\label{sec:analysis_fastpca}
This section entails detailed analysis for our proposed FAST-PCA algorithm. The first subsection is an auxiliary result followed by the main result in the following subsection.
\subsection{Auxiliary Result}\label{sec:analysis_centralized}
In this subsection, we provide an intermediary result that will help the analysis of our proposed algorithm. 
Let $\bC \in \R^{d\times d}$ be a covariance matrix whose eigenvectors are $\bq_l, l = 1,\ldots, d$, with corresponding eigenvalues $\lambda_l$. With an aim to estimate the first $K$ eigenvectors of $\bC$, we define a general update rule of the following form:
\begin{align}\nonumber
	&\bx_{g,k}^{(t+1)} = \bx_{g,k}^{(t)}  + \alpha\big(\bC\bx_{g,k}^{(t)} - \frac{(\bx_{g,k}^{(t)})^T\bC\bx_{g,k}^{(t)}}{\|\bx_{g,k}^{(t)}\|^2}\bx_{g,k}^{(t)} \\
	& \qquad \qquad \qquad \qquad \qquad \qquad \qquad  - \sum_{p=1}^{k-1}\bq_p\bq_p^T\bC\bx_{g,k}^{(t)}\big)\\\label{eq:centralized_sangerk}
	&= \bx_{g,k}^{(t)}  + \alpha\big(\bC\bx_{g,k}^{(t)} - \frac{(\bx_{g,k}^{(t)})^T\bC\bx_{g,k}^{(t)}}{\|\bx_{g,k}^{(t)}\|^2}\bx_{g,k}^{(t)} - \sum_{p=1}^{k-1}\lambda_p\bq_p\bq_p^T\bx_{g,k}^{(t)}\big),
\end{align}
where $\alpha$ is a constant step size. Note that this is not an algorithm in the true sense of the term as it cannot be implemented because of its dependence on the true eigenvectors $\bq_p$. The sole purpose of this update equation is to help in our ultimate goal of providing convergence guarantee for the FAST-PCA algorithm.

Since $\bq_l, l = 1, \ldots, d$, are the eigenvectors of a real symmetric matrix, they form a basis for $d$-dimensional space and can be used for expansion of any vector $\bx \in \R^d$. Let
\begin{equation}\label{eq:expansion_centralized}
	\tilde{\bx}_{g,k}^{(t)} = \frac{{\bx}_{g,k}^{(t)}}{\|{\bx}_{g,k}^{(t)}\|}=\sum_{l=1}^{d}z_{k,l}^{(t)}\bq_l,
\end{equation}
where $z_{k,l}^{(t)}$ is the coefficient corresponding to the eigenvector $\bq_l$ in the expansion of $\tilde{\bx}_{g,k}^{(t)}$. 
The update equation~\eqref{eq:centralized_sangerk} can be re-written as:
\begin{align}\nonumber
    \frac{{\bx}_{g,k}^{(t+1)}}{\|{\bx}_{g,k}^{(t+1)}\|} &= \Bigg(\frac{{\bx}_{g,k}^{(t)}}{\|{\bx}_{g,k}^{(t)}\|}  + \alpha\big(\bC\frac{{\bx}_{g,k}^{(t)}}{\|{\bx}_{g,k}^{(t)}\|} - \frac{(\bx_{g,k}^{(t)})^T\bC\bx_{g,k}^{(t)}}{\|\bx_{g,k}^{(t)}\|^2}\frac{{\bx}_{g,k}^{(t)}}{\|{\bx}_{g,k}^{(t)}\|} \\
    & \qquad \quad \quad - \sum_{p=1}^{k-1}\lambda_p\bq_p\bq_p^T\frac{{\bx}_{g,k}^{(t)}}{\|{\bx}_{g,k}^{(t)}\|}\big)\Bigg)\frac{\|{\bx}_{g,k}^{(t)}\|}{\|{\bx}_{g,k}^{(t+1)}\|}\\ \nonumber
    \tilde{\bx}_{g,k}^{(t+1)} &= \Bigg(\tilde{\bx}_{g,k}^{(t)}  + \alpha\big(\bC\tilde{\bx}_{g,k}^{(t)} - \frac{(\bx_{g,k}^{(t)})^T\bC\bx_{g,k}^{(t)}}{\|\bx_{g,k}^{(t)}\|^2}\tilde{\bx}_{g,k}^{(t)} \\ \label{eq:centralized_sangerk1}
    & \qquad \qquad \qquad \quad - \sum_{p=1}^{k-1}\lambda_p\bq_p\bq_p^T\tilde{\bx}_{g,k}^{(t)}\big)\Bigg)a_k^{(t)},
\end{align}
where $a_k^{(t)}= \frac{\|{\bx}_{g,k}^{(t)}\|}{\|{\bx}_{g,k}^{(t+1)}\|}$. Multiplying both sides of \eqref{eq:centralized_sangerk1} by $\bq_l^T$ and using the fact that $\bq_l^T\bq_{l'} = 0$ for $l \neq l'$, we get 
\begin{align*}
	{z}_{k,l}^{(t+1)} &= a_k^{(t)}\big({z}_{k,l}^{(t)} + \alpha(\bq_l^T\bC\tilde{\bx}_{g,k}^{(t)} - \bq_l^T(\sum_{p=1}^{k-1}\lambda_p\bq_p\bq_p^T\tilde{\bx}_{g,k}^{(t)}) \\
	& \qquad \qquad \qquad \qquad \qquad - (\tilde{\bx}_{g,k}^{(t)})^T\bC\tilde{\bx}_{g,k}^{(t)}{z}_{k,l}^{(t)})\big).
\end{align*}
This gives
\begin{align} \nonumber
    {z}_{k,l}^{(t+1)} &= a_k^{(t)}\big({z}_{k,l}^{(t)} - \alpha(\tilde{\bx}_{g,k}^{(t)})^T\bC\tilde{\bx}_{g,k}^{(t)}{z}_{k,l}^{(t)}\big), \\ \label{eq:eqn_z_lower}
    & \qquad \qquad \qquad \text{for} \quad l= 1, \ldots k - 1, \\ \nonumber
    \text{and} \quad {z}_{k,l}^{(t+1)} &= a_k^{(t)}\big({z}_{k,l}^{(t)} + \alpha(\lambda_l - (\tilde{\bx}_{g,k}^{(t)})^T\bC\tilde{\bx}_{g,k}^{(t)}){z}_{k,l}^{(t)}\big), \\ \label{eq:eqn_z_upper}
    & \qquad \qquad \qquad \text{for} \quad l= k, \ldots d.
\end{align}
In the following theorem, we show that $\bx_{g,k}^{(t)}$ converges to a multiple of the true eigenvector $\bq_k$ by proving convergence of the coefficients ${z}_{k,l}^{(t)}$ for $l=1,\ldots,d$.
\begin{theorem} \label{theorem:convergence_centralized}
    Suppose $\bC$ has $K$ distinct eigenvalues, i.e., $\lambda_1 > \lambda_2 > \cdots > \lambda_K > \lambda_{K+1} \geq \cdots \geq \lambda_d \geq 0$ and $\alpha < \frac{1}{\lambda_1}$, $\bq_k^T\bx_{g,k}^{(0)} \neq 0$, and $\|\bx_{g,k}^{(0)}\| = 1$ for all $k = 1,\ldots, K$. Then the update equation for $\bx_{g,k}^{(t)}$ given by \eqref{eq:centralized_sangerk} converges at a linear rate to a multiple of the eigenvector $\pm\bq_k$ corresponding to the $k^{th}$ largest eigenvalue $\lambda_k$ of the covariance matrix $\bC$ for $k=1,\ldots,K$.
\end{theorem}
\begin{proof} We prove the linear convergence of $\bx_{g,k}^{(t)}$ to a multiple of $\bq_k$ by proving that $\tilde{\bx}_{g,k}^{(t)}$ converges to $\bq_k$ at a linear rate. The convergence of $\tilde{\bx}_{g,k}^{(t)}$ to $\bq_k$ requires convergence of the lower-order coefficients $z_{k,1}^{(t)}, \ldots, z_{k,k-1}^{(t)}$ and the higher-order coefficients $z_{k,k+1}^{(t)}, \ldots, z_{k,d}^{(t)}$ to 0 and convergence of $z_{k,k}^{(t)}$ to $\pm 1$. 
To this end, Lemma~\ref{lemma:coeff_decay_lower} in the appendix proves linear convergence of the lower-order coefficients $z_{k,1}^{(t)}, \ldots, z_{k,k-1}^{(t)}$ to 0 by showing $\sum_{l=1}^{k-1}(z_{k,l}^{(t)})^2 \leq a_1\gamma_k^{t}$ for some constants $a_1 > 0, \gamma_k = \frac{1}{1+\alpha\lambda_k} < 1$. Furthermore, Lemma~\ref{lemma:coeff_decay_upper} in the appendix shows that $\sum_{l=k+1}^{d}({z}_{k,l}^{(t)})^2 \leq a_2\delta_k^{t}$, where $a_2 > 0$ and $\delta_k = \frac{1+\alpha\lambda_{k+1}}{1+\alpha\lambda_k} < 1$, thereby proving linear convergence of the higher-order coefficients to 0. 
The formal statements and proofs of Lemma~\ref{lemma:coeff_decay_lower} and Lemma~\ref{lemma:coeff_decay_upper} are given in Appendix~\ref{app:lemma3} and Appendix~\ref{app:lemma4}, respectively. Finally, since $\|\tilde{\bx}_{g,k}^{(t)}\|^2=1$, we have 
\begin{align*}
    \sum_{l=1}^d (z_{k,l}^{(t)})^2 &= 1 \\
    \text{or,}\quad 1 - (z_{k,k}^{(t)})^2 &= \sum_{l=1}^{k-1} (z_{k,l}^{(t)})^2 + \sum_{l=k+1}^d (z_{k,l}^{(t)})^2\\
    &\leq a_1\gamma_k^{t} + a_2\delta_k^{t} \\
    &< a_3\delta_k^{t}, 
\end{align*}
where $a_3 = \max\{a_1,a_2\}$ and $\delta_k = \max\{\gamma_k,\delta_k\}$. This shows $(z_{k,k}^{(t)})^2$ converges to 1 and $(z_{l,k}^{(t)})^2, l \neq k$, converges to 0 at a linear rate of $\cO(\delta_k^t)$ where $\delta_k = \frac{1+\alpha\lambda_{k+1}}{1+\alpha\lambda_k}$. Thus, $\tilde{\bx}_{g,k}^{(t)} \rightarrow \pm \bq_k$ and $(\tilde{\bx}_{g,k}^{(t)})^T\bC\tilde{\bx}_{g,k}^{(t)} \rightarrow \lambda_k$. We also know from~\eqref{eq:centralized_sangerk} that
\begin{align}\nonumber
	&\bx_{g,k}^{(t+1)} = \bx_{g,k}^{(t)}  + \alpha\big((\bC -  \sum_{p=1}^{k-1}\lambda_p\bq_p\bq_p^T)\bx_{g,k}^{(t)} - \frac{(\bx_{g,k}^{(t)})^T\bC\bx_{g,k}^{(t)}}{\|\bx_{g,k}^{(t)}\|^2}\bx_{g,k}^{(t)} \big)\\\nonumber
	&\text{i.e.,}\quad \|\bx_{g,k}^{(t+1)}\|^2 = \|\bx_{g,k}^{(t)}\|^2  + \alpha^2\|\big((\bC -  \sum_{p=1}^{k-1}\lambda_p\bq_p\bq_p^T){\bx}_{g,k}^{(t)}  \\\nonumber
	& \qquad \qquad \qquad \qquad \qquad - \frac{(\bx_{g,k}^{(t)})^T\bC\bx_{g,k}^{(t)}}{\|\bx_{g,k}^{(t)}\|^2}{\bx}_{g,k}^{(t)} \big)\|^2 - \\ \nonumber
	& 2\alpha(\bx_{g,k}^{(t)})^T\big((\bC - \sum_{p=1}^{k-1}\lambda_p\bq_p\bq_p^T){\bx}_{g,k}^{(t)} - \frac{(\bx_{g,k}^{(t)})^T\bC\bx_{g,k}^{(t)}}{\|\bx_{g,k}^{(t)}\|^2}{\bx}_{g,k}^{(t)} \big) \\\nonumber
	&= \|\bx_{g,k}^{(t)}\|^2  + \alpha^2\|\big((\bC -  \sum_{p=1}^{k-1}\lambda_p\bq_p\bq_p^T){\bx}_{g,k}^{(t)} - \\ \nonumber 
	& \qquad \quad \frac{(\bx_{g,k}^{(t)})^T\bC\bx_{g,k}^{(t)}}{\|\bx_{g,k}^{(t)}\|^2}{\bx}_{g,k}^{(t)} \big)\|^2  + 2\alpha \sum_{p=1}^{k-1}\lambda_p({\bx}_{g,k}^{(t)})^T\bq_p\bq_p^T{\bx}_{g,k}^{(t)} \\ \nonumber
	&= \|\bx_{g,k}^{(t)}\|^2  + \alpha^2\|(\bC -  \sum_{p=1}^{k-1}\lambda_p\bq_p\bq_p^T)\tilde{\bx}_{g,k}^{(t)}\|{\bx}_{g,k}^{(t)}\|   \\ \nonumber
	& \qquad \qquad \qquad - \frac{(\bx_{g,k}^{(t)})^T\bC\bx_{g,k}^{(t)}}{\|\bx_{g,k}^{(t)}\|^2}\tilde{\bx}_{g,k}^{(t)}\|{\bx}_{g,k}^{(t)}\| \|^2  \\ \label{eq:general_kras}
	& \qquad \qquad \qquad + 2\alpha \|{\bx}_{g,k}^{(t)}\|^2\sum_{p=1}^{k-1}\lambda_p(\tilde{\bx}_{g,k}^{(t)})^T\bq_p\bq_p^T\tilde{\bx}_{g,k}^{(t)} 
\end{align}
As $\tilde{\bx}_{g,k}^{(t)} \rightarrow \pm \bq_k$ and $\frac{(\bx_{g,k}^{(t)})^T\bC\bx_{g,k}^{(t)}}{\|\bx_{g,k}^{(t)}\|^2} \rightarrow \lambda_k$, we have 
\begin{align*}
    (\bC -  \sum_{p=1}^{k-1}\lambda_p\bq_p\bq_p^T)\tilde{\bx}_{g,k}^{(t)}\|{\bx}_{g,k}^{(t)}\| \rightarrow \pm\bC\bq_k\|{\bx}_{g,k}^{(t)}\| = \pm \lambda_k\bq_k\|{\bx}_{g,k}^{(t)}\|. 
\end{align*}
and,
\begin{align*}
    \sum_{p=1}^{k-1}\lambda_p(\tilde{\bx}_{g,k}^{(t)})^T\bq_p\bq_p^T\tilde{\bx}_{g,k}^{(t)} \rightarrow 0.
\end{align*}
Thus from~\eqref{eq:general_kras}, we get
\begin{align*}
    \|\bx_{g,k}^{(t+1)}\|^2 - \|\bx_{g,k}^{(t)}\|^2  \rightarrow 0,
\end{align*}
which implies $\|\bx_{g,k}^{(t)}\|$ converges to some constant $c_k > 0$ which further implies ${\bx}_{g,k}^{(t)} \rightarrow \pm c_k\bq_k$.\end{proof}
\subsection{Analysis of FAST-PCA}
In this subsection, we provide a detailed analysis proving that the FAST-PCA algorithm converges at a linear rate to the true eigenvectors $\bq_k, k = 1,\ldots,K$, of the global covariance matrix $\bC$. Specifically, let $\bX_i^{(t)} = \begin{bmatrix}\bx_{i,1}^{(t)},\ldots,\bx_{i,K}^{(t)}\end{bmatrix} \in \R^{d\times K}$ be the estimate of the $K$ eigenvectors at node $i$, then we show that square of the sine of the angle between $\bx_{i,k}^{(t)}, \forall i = 1,\ldots,M$, and $\bq_k$ for $k=1,\ldots,K$ converges to 0 at the rate of $\cO(\rho^t)$ for some $\rho \in (0,1)$. 

We know from Theorem~\ref{theorem:convergence_centralized} in the previous section that for estimation of the $k^{th}$ eigenvector, any general iterate of the form
\begin{equation}\label{eq:centralized_adsa_kras}
    \bx_{g,k}^{(t+1)} = \bx_{g,k}^{(t)} +  \alpha\Big(\bC\bx_{g,k}^{(t)} - \frac{(\bx_{g,k}^{(t)})^T\bC\bx_{g,k}^{(t)}}{\|\bx_{g,k}^{(t)}\|^2}\bx_{g,k}^{(t)} - \sum_{p=1}^{k-1}\bq_p\bq_p^T\bC\bx_{g,k}^{(t)}\Big)
\end{equation}
converges at a linear rate to a scalar multiple of the eigenvector $\bq_k$ of $\bC$ if the top $K$ eigenvalues of $\bC$ are distinct, i.e., $\lambda_1 > \lambda_2 > \cdots > \lambda_K > \lambda_{K+1} \geq \cdots \geq \lambda_d \geq 0$ as well as if $\alpha < \frac{1}{\lambda_1}$ and $\bq_k^T\bx_{g,k}^{(0)} \neq 0$. Specifically, $\bx_{g,k}^{(t)}$ converges to either $c_k\bq_k$ or $-c_k\bq_k$ at a linear rate in this case. Mathematically,
\begin{align}\nonumber
	& \|\bx_{g,k}^{(t+1)} - \bx_k^*\| \leq \delta_k\|\bx_{g,k}^{(t)} - \bx_k^*\|, \quad \text{for} \\ \label{eq:convergence_kras}
	& \quad 0< \delta_k = \frac{1+\alpha\lambda_{k+1}}{1+\alpha\lambda_{k}} < 1 \quad \text{and} \quad \bx_k^* = c_k\bq_k \quad \text{or}\quad -c_k\bq_k.
\end{align}
Now, if $\bW = [w_{ij}]$ is the weight matrix underlying the graph representing the network, then the iterates of FAST-PCA for the estimation of the $k^{th}$ eigenvector are given as follows:
\begin{align}\label{eq:edsa_xi}
	\bx_{i,k}^{(t+1)} &= \frac{1}{2}\bx_{i,k}^{(t)} + \sum_{j\in \cN_i}\frac{w_{ij}}{2}\bx_{j,k}^{(t)} + \alpha\bs_{i,k}^{(t)}\\ \label{eq:edsa_si}
	\bs_{i,k}^{(t+1)} &= \frac{1}{2}\bs_{i,k}^{(t)}+ \sum_{j\in \cN_i}\frac{w_{ij}}{2}\bs_{j,k}^{(t)} + \bh_i(\bx_{i,k}^{(t+1)}) - \bh_i(\bx_{i,k}^{(t)}), 
\end{align}
where $\bx_{i,k}^{(t)}$ is the estimate of the $k^{th}$ eigenvector, $\bh_i(\bx_{i,k}^{(t)})$ is the pseudo-gradient given as $\bh_i(\bx_{i,k}^{(t)}) = \bC_i\bx_{i,k}^{(t)} - \frac{(\bx_{i,k}^{(t)})^T\bC_i\bx_{i,k}^{(t)}}{\|\bx_{i,k}^{(t)}\|^2}\bx_{i,k}^{(t)} - \sum_{p=1}^{k-1}\frac{\bx_{i,p}^{(t)}(\bx_{i,p}^{(t)})^T\bC_i}{\|\bx_{i,p}^{(t)}\|^2}\bx_{i,k}^{(t)}$ and $\bs_{i,k}^{(t)}$ is the estimate of the average pseudo-gradients. Let us define the following stacked versions of the quantities $\bx_{i,k}^{(t)}, \bs_{i,k}^{(t)}$, and $\bh_i(\bx_{i,k}^{(t)})$ for  $i = 1,\ldots,M$ as
\[
\bx_{k}^{(t)} = \begin{bmatrix}
	\bx_{1,k}^{(t)}\\
	\bx_{2,k}^{(t)}\\
	\vdots\\
	\bx_{M,k}^{(t)}\\
\end{bmatrix} ,\quad
\bh(\bx_{k}^{(t)})  = \begin{bmatrix}
	\bh_1(\bx_{1,k}^{(t)})\\
	\bh_2(\bx_{2,k}^{(t)})\\
	\vdots\\
	\bh_M(\bx_{M,k}^{(t)})\\
\end{bmatrix} ,  \quad
\bs_k^{(t)}  = \begin{bmatrix}
	\bs_{1,k}^{(t)}\\
	\bs_{2,k}^{(t)}\\
	\vdots \\	
	\bs_{M,k}^{(t)}\\
\end{bmatrix}. \]
Let $\bar{\bx}_{k}^{(t)}$ and $\bar{\bs}_{k}^{(t)}$ denote the average of $\{\bx_{i,k}^{(t)}\}_{i=1}^M$ and $\{\bs_{i,k}^{(t)}\}_{i=1}^M$, respectively. Taking the average of~\eqref{eq:edsa_xi} and~\eqref{eq:edsa_si} over all nodes $i=1,\ldots,M$, we get 
\begin{align}\label{eq:edsa_xavg}
	\frac{1}{M}\sum_{i=1}^{M}\bx_{i,k}^{(t+1)} &= \bar{\bx}_{k}^{(t+1)} = \bar{\bx}_{k}^{(t)} + \alpha\bar{\bs}_{k}^{(t)}\\ \label{eq:edsa_savg}
	\frac{1}{M}\sum_{i=1}^{M}\bs_{i,k}^{(t+1)} &= \bar{\bs}_{k}^{(t+1)} = \bar{\bs}_{k}^{(t)} + \bg(\bx_{k}^{(t+1)}) - \bg(\bx_{k}^{(t)}) = \bg(\bx_{k}^{(t+1)}),
\end{align}
where $\bg(\bx_k^{(t)}) = \frac{1}{M}\sum_{i=1}^{M} \bh_i(\bx_{i,k}^{(t)}) \in \R^d$. 
Additionally, we also define the following stacked versions (denoted by subscript `s') such that all these are in $\R^{Md}$:
\[
\bar{\bx}_{s,k}^{(t)}  = \begin{bmatrix}
	\bar{\bx}_{k}^{(t)}\\
	\bar{\bx}_{k}^{(t)}\\
	\vdots \\
	\bar{\bx}_{k}^{(t)}\\
\end{bmatrix},\quad
\bar{\bs}_{s,k}^{(t)}  = \begin{bmatrix}
	\bar{\bs}_{k}^{(t)}\\
	\bar{\bs}_{k}^{(t)}\\
	\vdots\\
	\bar{\bs}_{k}^{(t)}\\
\end{bmatrix},\quad
\bg_{s}(\bx_k^{(t)}) = \begin{bmatrix}
	\bg(\bx_{k}^{(t)})\\
	\bg(\bx_{k}^{(t)})\\
	\vdots\\
	\bg(\bx_{k}^{(t)})\\
\end{bmatrix}.
\]
Using these definitions,~\eqref{eq:edsa_xi} and~\eqref{eq:edsa_si} can be re-written as
\begin{align} \label{eq:adsax_stacked}
	\bx_k^{(t+1)} &= \frac{1}{2}((\bI_M+\bW)\otimes\bI_d)\bx_k^{(t)} + \alpha \bs_k^{(t)}, \\ \label{eq:adsas_stacked}
	\bs_k^{(t+1)} &= \frac{1}{2}((\bI_M+\bW)\otimes\bI_d)\bs_k^{(t)} + \bh(\bx_k^{(t+1)}) - \bh(\bx_k^{(t)}).
\end{align}
Also,
\begin{align}
	\label{eq:adsas_stacked_avg}
	\bar{\bs}_{s,k}^{(t+1)} &= \bar{\bs}_{s,k}^{(t)} + \bg_s(\bx_k^{(t+1)}) - \bg_s(\bx_k^{(t)}) = \bg_s(\bx_k^{(t+1)})\\
	\label{eq:adsax_stacked_avg}
	\bar{\bx}_{s,k}^{(t+1)} &= \bar{\bx}_{s,k}^{(t)} + \alpha \bar{\bs}_{s,k}^{(t)} = \bar{\bx}_{s,k}^{(t)} + \alpha \bg_s(\bx_k^{(t)})\\\nonumber
	{\bg}(\bar{\bx}_{s,k}^{(t)}) &= \frac{1}{M}\sum_{i=1}^{M} \Big(\bC_i\bar{\bx}_{k}^{(t)} - \frac{(\bar{\bx}_{k}^{(t)})^T\bC_i\bar{\bx}_{k}^{(t)}}{\|\bar{\bx}_{k}^{(t)}\|^2}\bar{\bx}_{k}^{(t)}- \\\nonumber
	&\qquad \qquad \qquad \qquad \sum_{p=1}^{k-1}\frac{\bx_{i,p}^{(t)}(\bx_{i,p}^{(t)})^T\bC_i}{\|\bx_{i,p}^{(t)}\|^2}\bar{\bx}_{k}^{(t)}\Big) \\ \nonumber
	&= \frac{1}{M}\Big(\bC\bar{\bx}_{k}^{(t)} - \frac{(\bar{\bx}_{k}^{(t)})^T\bC\bar{\bx}_{k}^{(t)}}{\|\bar{\bx}_{k}^{(t)}\|^2}\bar{\bx}_{k}^{(t)} - \\\label{eq:avg_g}
	&\qquad \qquad \qquad \quad \sum_{i=1}^{M} \sum_{p=1}^{k-1}\frac{\bx_{i,p}^{(t)}(\bx_{i,p}^{(t)})^T\bC_i}{\|\bx_{i,p}^{(t)}\|^2}\bar{\bx}_{k}^{(t)}\Big).
\end{align}
Now, we first show that $\bx_{i,1}^{(t)}$ converges to a multiple of $\bq_1$ at a linear rate and then proceed with the proof for $k=2,\ldots,K$ through induction.

\textbf{\textit{Case I for Induction -- $k=1$}:} The iterates of FAST-PCA for estimation of the dominant eigenvector are
\begin{align}\label{eq:edsa_xi1}
	\bx_{i,1}^{(t+1)} &= \frac{1}{2}\bx_{i,1}^{(t)} + \sum_{j\in \cN_i}\frac{w_{ij}}{2}\bx_{j,1}^{(t)} + \alpha\bs_{i,1}^{(t)}\\ \label{eq:edsa_si1}
	\bs_{i,1}^{(t+1)} &= \frac{1}{2}\bs_{i,1}^{(t)}+ \sum_{j\in \cN_i}\frac{w_{ij}}{2}\bs_{j,1}^{(t)} + \bh_i(\bx_{i,1}^{(t+1)}) - \bh_i(\bx_{i,1}^{(t)}) 
\end{align}
where $\bh_i(\bx_{i,1}^{(t)}) = \bC_i\bx_{i,1}^{(t)} - \frac{(\bx_{i,1}^{(t)})^T\bC_i\bx_{i,1}^{(t)}}{\|\bx_{i,1}^{(t)}\|^2}\bx_{i,1}^{(t)}$.
\begin{lemma}\label{lemma:lipshitz_continuous}
    The function $\bh_i: \R^d \rightarrow \R^d$ with $\bh_i(\bv) = \bC_i\bv - \frac{(\bv)^T\bC_i\bv}{\|\bv\|^2}\bv$ is Lipschitz continuous with Lipschitz constant $L_1 = 6\lambda_1$.
\end{lemma}
The proof of this lemma is deferred to Appendix~\ref{app:lemma_lipschitz}. For Lipschitz continuous functions $\bh$ and $\bg(\bx_1)$ defined above, the following lemma holds true, the proof of which is deferred to Appendix~\ref{app:lemma_inequalities}.

\begin{lemma}\label{lemma:inequalities}
	The following inequalities hold for $L_1 = 6\lambda_1$:
	\begin{enumerate}
		\item $\|\bh(\bx_1^{(t)}) - \bh(\bx_1^{(t-1)})\|_2 \leq L_1\|\bx_1^{(t)} - \bx_1^{(t-1)}\|_2$
		\item $\|\bg(\bx_{1}^{(t)}) - \bg(\bx_{1}^{(t-1)})\|_2 \leq \frac{L_1}{\sqrt{M}}\|\bx_1^{(t)} - \bx_1^{(t-1)}\|_2$
		\item $\|\bg(\bx_{1}^{(t)}) - {\bg}(\bar{\bx}_{s,1}^{(t)})\|_2 \leq \frac{L_1}{\sqrt{M}}\|\bx_1^{(t)} - \bar{\bx}_{s,1}^{(t)}\|_2$
	\end{enumerate}
\end{lemma}
These inequalities aid our main theorem presented next that proves the convergence of the iterate $\bx_{i,1}^{(t)}$ at node $i$ to $\bx_1^* = \pm c_1\bq_1$, where $c_1$ is a constant. 
\begin{theorem}\label{theorem:theorem1}
	Suppose $\alpha < \frac{\lambda_1-\lambda_2}{42}(\frac{1-\beta}{9\lambda_1})^2$, where $\lambda_1, \lambda_2$ are the largest and second-largest eigenvalues of $\bC$, $\beta = \max \{|\lambda_2(\bW)|,|\lambda_M(\bW)|\}$ is the second largest absolute eigenvalue of $\bW$, $\bq_1^T\bar{\bx}_{1}^{(0)} \neq 0$, and the graph underlying the network is connected. Then the estimate $\bx_{i,1}^{(t)}$ from FAST-PCA converges to the eigenvector $\pm c_1\bq_1$ corresponding to the largest eigenvalue $\lambda_1$ of $\bC$ at each node $i = 1,\ldots,M$ at a linear rate.
\end{theorem}
\begin{proof}
For proving the convergence of $\bx_{i,1}^{(t)}$ to $\bx_1^*=\pm c_1\bq_1$, $\forall i = 1,\ldots,M$, we prove that the distance of average $\bar{\bx}_1^{(t)}$ from $\bx_1^*$, the consensus error as well as the distance of $\bs_{i,1}^{(t)}$ from the average pseudo-gradient $\bg(\bx_{1}^{(t)})$ decay to zero at a linear rate.
From~\eqref{eq:adsas_stacked}, we have
\begin{align*}
	\bs_1^{(t)} - \bg_s(\bx_{1}^{(t)}) &= \frac{1}{2}((\bI_M+\bW)\otimes\bI_d)\bs_1^{(t-1)} - \bg_s(\bx_{1}^{(t-1)}) + \\
	& \bh(\bx_1^{(t)}) - \bh(\bx_1^{(t-1)}) - (\bg_s(\bx_{1}^{(t)}) - \bg_s(\bx_{1}^{(t-1)})).
\end{align*}
From the definitions of $\bar{\bs}_1^{(t)}$, $\bg(\bx_{1}^{(t-1)})$ and $\bg_s(\bx_{1}^{(t-1)})$, it is obvious that $(\frac{1}{M}\bone\bone^T\otimes\bI_d)\bs_1^{(t-1)} = \bone\bar{\bs}_1^{(t-1)}=\bar{\bs}_{s,1}^{(t-1)}=\bg_s(\bx_{1}^{(t-1)})$. Thus, 
\begin{align}\nonumber
	&\|\bs_1^{(t)} - \bg_s(\bx_{1}^{(t)})\| \\\nonumber
	&= \|\frac{1}{2}((\bI_M+\bW)\otimes\bI_d)\bs_1^{(t-1)} - \bg_s(\bx_{1}^{(t-1)}) + \bh(\bx_1^{(t)}) - \bh(\bx_1^{(t-1)}) \\\nonumber
	&  \qquad \qquad \qquad \qquad \qquad - (\bg_s(\bx_{1}^{(t)}) - \bg_s(\bx_{1}^{(t-1)}))\|\\ \nonumber
	&= \|\frac{1}{2}((\bI_M+\bW)\otimes\bI_d)\bs_1^{(t-1)} -( \frac{1}{M}\bone\bone^T\otimes\bI_d)\bs_1^{(t-1)} + \bg_s(\bx_{1}^{(t-1)}) \\\nonumber
	&  - \bg_s(\bx_{1}^{(t-1)}) + \bh(\bx_1^{(t)}) - \bh(\bx_1^{(t-1)}) - (\bg_s(\bx_{1}^{(t)}) - \bg_s(\bx_{1}^{(t-1)}))\|\\ \nonumber
	&= \|(\frac{1}{2}((\bI_M+\bW)\otimes\bI_d) -  \frac{1}{M}\bone\bone^T\otimes\bI_d)(\bs_1^{(t-1)} - \bg_s(\bx_{1}^{(t-1)})) + \\\nonumber
	& \qquad \qquad\bh(\bx_1^{(t)}) - \bh(\bx_1^{(t-1)}) - (\bg_s(\bx_{1}^{(t)}) - \bg_s(\bx_{1}^{(t-1)}))\|\\ \nonumber
	&\leq \|((\frac{1}{2}(\bI_M+\bW) -  \frac{1}{M}\bone\bone^T)\otimes\bI_d)(\bs_1^{(t-1)} - \bg_s(\bx_{1}^{(t-1)}))\| + \\\label{eq:step1}
	& \qquad \qquad \|\bh(\bx_1^{(t)}) - \bh(\bx_1^{(t-1)}) - (\bg_s(\bx_{1}^{(t)}) - \bg_s(\bx_{1}^{(t-1)}))\|.
\end{align}
Next, we simplify the second term of the above inequality~\eqref{eq:step1} as follows:
\begin{align*}
	&\|\bh(\bx_1^{(t)}) - \bh(\bx_1^{(t-1)}) - (\bg_s(\bx_{1}^{(t)}) - \bg_s(\bx_{1}^{(t-1)}))\|^2 \\
	=& \|\bh(\bx_1^{(t)}) - \bh(\bx_1^{(t-1)})\|^2 + \|\bg_s(\bx_{1}^{(t)}) - \bg_s(\bx_{1}^{(t-1)})\|^2  \\
	& - 2\langle \bh(\bx_1^{(t)}) - \bh(\bx_1^{(t-1)}), \bg_s(\bx_{1}^{(t)}) - \bg_s(\bx_{1}^{(t-1)})\rangle \\
	=& \|\bh(\bx_1^{(t)}) - \bh(\bx_1^{(t-1)})\|^2 + \|\bg_s(\bx_{1}^{(t)}) - \bg_s(\bx_{1}^{(t-1)})\|^2 \\
	& - 2\sum_{i=1}^{M}\langle \bh_i(\bx_{i,1}^{(t)}) - \bh_i(\bx_{i,1}^{(t-1)}), \bg(\bx_{1}^{(t)}) - \bg(\bx_{1}^{(t-1)})\rangle \\
	=& \|\bh(\bx_1^{(t)}) - \bh(\bx_1^{(t-1)})\|^2 + M\|\bg(\bx_{1}^{(t)}) - \bg(\bx_{1}^{(t-1)})\|^2  \\
	&- 2M\langle \bg(\bx_{1}^{(t)}) - \bg(\bx_{1}^{(t-1)}), \bg(\bx_{1}^{(t)}) - \bg(\bx_{1}^{(t-1)})\rangle \\
	=& \|\bh(\bx_1^{(t)}) - \bh(\bx_1^{(t-1)})\|^2 - M\|\bg(\bx_{1}^{(t)}) - \bg(\bx_{1}^{(t-1)})\|^2  \\ 
	\leq&  \|\bh(\bx_1^{(t)}) - \bh(\bx_1^{(t-1)})\|^2.
\end{align*}
Thus,
\begin{align}\nonumber
&\|\bh(\bx_1^{(t)}) - \bh(\bx_1^{(t-1)}) - (\bg_s(\bx_{1}^{(t)}) - \bg_s(\bx_{1}^{(t-1)}))\| \\ \label{eq:interim_step1}
&\leq \|\bh(\bx_1^{(t)}) - \bh(\bx_1^{(t-1)})\| \leq L_1\|\bx_1^{(t)} - \bx_1^{(t-1)}\|.
\end{align}
From~\eqref{eq:step1} and~\eqref{eq:interim_step1}, we have the following
\begin{align}\nonumber
	&\qquad\|\bs_1^{(t)}- \bg_s(\bx_{1}^{(t)})\| \\\nonumber
	&\leq \|((\frac{1}{2}(\bI_M+\bW) -  \frac{1}{M}\bone\bone^T)\otimes\bI_d)(\bs_1^{(t-1)} - \bg_s(\bx_{1}^{(t-1)}))\| \\ \nonumber
	& \qquad \qquad \qquad \qquad \qquad \qquad \qquad + L_1\|\bx_1^{(t)} - \bx_1^{(t-1)}\|\\ \label{eq:eqn_sys11}
	&\leq \frac{1+\beta}{2}\|\bs_1^{(t-1)} - \bg_s(\bx_{1}^{(t-1)})\| + L_1\|\bx_1^{(t)} - \bx_1^{(t-1)}\|,
\end{align}
where $\beta$ is absolute value of the second largest eigenvalue of the weight matrix $\bW$ i.e., $\beta = \max \{|\lambda_2(\bW)|,|\lambda_M(\bW)|\}$. As pointed out before, network connectivity ensures that $\beta < 1$. Next, from~\eqref{eq:adsax_stacked} and~\eqref{eq:adsax_stacked_avg}, we have 
\begin{align*}
	\bx_1^{(t)} - \bar{\bx}_{s,1}^{(t)} &= \frac{1}{2}((\bI_M+\bW)\otimes\bI_d)\bx_1^{(t-1)} - \bar{\bx}_{s,1}^{(t-1)} + \\
	& \qquad \qquad \qquad \qquad \qquad \alpha(\bs_1^{(t-1)} - \bg_s(\bx_{1}^{(t-1)}))\\
	&= ((\frac{1}{2}(\bI_M+\bW) - \frac{1}{M}\bone\bone^T)\otimes\bI_d)(\bx_1^{(t-1)} - \bar{\bx}_{s,1}^{(t-1)}) \\
	& \qquad \qquad \qquad \qquad \qquad + \alpha(\bs_1^{(t-1)} - \bg_s(\bx_{1}^{(t-1)})).
\end{align*}
Thus,
\begin{equation} \label{eq:eqn_sys2}
	\|\bx_1^{(t)} - \bar{\bx}_{s,1}^{(t)}\| \leq \frac{1+\beta}{2}\|\bx_1^{(t-1)} - \bar{\bx}_{s,1}^{(t-1)}\| + \alpha\|\bs_1^{(t-1)} - \bg_s(\bx_{1}^{(t-1)})\|.
\end{equation}
Next, we bound $\|\bar{\bx}_{1}^{(t)} - \bx_1^*\|$. We know from~\eqref{eq:edsa_xavg}
\begin{eqnarray*}
	\bar{\bx}_{1}^{(t)} &=& \bar{\bx}_{1}^{(t-1)} + \alpha\bar{\bs}_{1}^{(t-1)} = \bar{\bx}_{1}^{(t-1)} + \alpha\bg(\bx_{1}^{(t-1)}) \\
	&=& \bar{\bx}_{1}^{(t-1)} + \alpha{\bg}(\bar{\bx}_{s,1}^{(t-1)}) + \alpha(\bg(\bx_{1}^{(t-1)}) - {\bg}(\bar{\bx}_{s,1}^{(t-1)}))\\
	&=& \bar{\bx}_{1}^{(t-1)} + \frac{\alpha}{M}(\bC\bar{\bx}_1^{(t-1)} - \frac{(\bar{\bx}_1^{(t-1)})^T\bC\bar{\bx}_1^{(t-1)}}{\|\bar{\bx}_1^{(t-1)}\|^2}\bar{\bx}_1^{(t-1)}) \\
	&& \qquad \qquad \qquad \qquad + \alpha(\bg(\bx_{1}^{(t-1)}) - {\bg}(\bar{\bx}_{s,1}^{(t-1)})).
\end{eqnarray*}
Using~\eqref{eq:centralized_adsa_kras} and~\eqref{eq:convergence_kras}, we know that an iterate of the form
\begin{align*}
    \bar{\bx}_{1}^{(t)} = \bar{\bx}_{1}^{(t-1)} +  \alpha(\bC\bar{\bx}_{1}^{(t-1)} - \frac{(\bar{\bx}_{1}^{(t-1)})^T\bC\bar{\bx}_{1}^{(t-1)}}{\|\bar{\bx}_{1}^{(t-1)}\|^2}\bar{\bx}_{1}^{(t-1)})
\end{align*}
converges linearly as
\begin{align*}
    \|\bar{\bx}_{1}^{(t)} - \bx_1^*\| \leq \delta_1\|\bar{\bx}_{1}^{(t-1)} - \bx_1^*\|,
\end{align*}
where $\bx_1^* = \pm c_1\bq_1$ and $\delta_1 = \frac{1+\alpha\lambda_2}{1+\alpha\lambda_1}$. Thus,
\begin{align} \nonumber
	\|\bar{\bx}_{1}^{(t)} - \bx_1^*\| &\leq \delta_1\| \bar{\bx}_{1}^{(t-1)} - \bx_1^*\| + \alpha \|\bg(\bx_{1}^{(t-1)}) - {\bg}(\bar{\bx}_{s,1}^{(t-1)})\|\\ \label{eq:eqn_sys3}
	&\leq \delta_1\|\bar{\bx}_{1}^{(t-1)} - \bx_1^*\| + \alpha\frac{L_1}{\sqrt{M}} \|\bx_1^{(t-1)} - \bar{\bx}_{s,1}^{(t-1)}\|.
\end{align}

Now we will bound $\|\bx_1^{(t)} - \bx_1^{(t-1)}\|$. We know from~\eqref{eq:avg_g}
\begin{align*}
    {\bg}(\bar{\bx}_{s,1}^{(t)}) = \frac{1}{M}(\bC\bar{\bx}_{1}^{(t)} - \frac{(\bar{\bx}_{1}^{(t)})^T\bC\bar{\bx}_{1}^{(t)}}{\|\bar{\bx}_{1}^{(t)}\|^2}\bar{\bx}_{1}^{(t)}). 
\end{align*}
Thus ${\bg}(\bx_{s,1}^*) = \frac{1}{M}(\bC{\bx}_{1}^* - \frac{({\bx}_{1}^*)^T\bC{\bx}_{1}^*}{\|{\bx}_{1}^*\|^2}{\bx}_{1}^*) = 0$, where $\bx_{s,1}^* = \begin{bmatrix}
 (\bx_{1}^*)^\tT,\ldots, (\bx_{1}^*)^\tT
\end{bmatrix}^\tT$. Hence, 
\begin{align*}
	\|{\bg}_s(\bar{\bx}_{s,1}^{(t-1)})\| = \sqrt{M}\|{\bg}(\bar{\bx}_{s,1}^{(t-1)})\| &= \sqrt{M}\|{\bg}(\bar{\bx}_{s,1}^{(t-1)}) - {\bg}({\bx}_{s,1}^*)\| \\
	& \leq L_1\sqrt{M}\|\bar{\bx}_{1}^{(t-1)} - \bx_1^*\|.
\end{align*}
Using the above inequality and Lemma~\ref{lemma:inequalities}, we get
\begin{align}\nonumber
	\|\bs_1^{(t-1)}\| &= \|\bs_1^{(t-1)} - \bg_s(\bx_{1}^{(t-1)}) + \bg_s(\bx_{1}^{(t-1)}) - {\bg}_s(\bar{\bx}_{s,1}^{(t-1)}) \\
	& \qquad \qquad \qquad \qquad \qquad \qquad + {\bg}_s(\bar{\bx}_{s,1}^{(t-1)})\|\\\nonumber
	&\leq \|\bs_1^{(t-1)} - \bg_s(\bx_{1}^{(t-1)})\| + \|\bg_s(\bx_{1}^{(t-1)}) - {\bg}_s(\bar{\bx}_{s,1}^{(t-1)})\| \\
	& \qquad \qquad \qquad \qquad \qquad \qquad + \|{\bg}_s(\bar{\bx}_{s,1}^{(t-1)})\|\\\nonumber
	&\leq \|\bs_1^{(t-1)} - \bg_s(\bx_{1}^{(t-1)})\| + L_1\|\bx_1^{(t-1)} - \bar{\bx}_{s,1}^{(t-1)}\| \\\label{eq:bound_s1}
	& \qquad \qquad \qquad \qquad + L_1\sqrt{M}\|\bar{\bx}_{1}^{(t-1)} - \bx_1^*\|.
\end{align}
Thus,
\begin{align*}\nonumber
	\|\bx_1^{(t)} - \bx_1^{(t-1)}\| &= \|\frac{1}{2}((\bI_M+\bW)\otimes\bI_d)\bx_1^{(t-1)} - \bx_1^{(t-1)} + \alpha\bs_1^{(t-1)}\| \\\nonumber
	&=  \|(\frac{1}{2}((\bI_M+\bW)\otimes\bI_d) - \bI_{Md})(\bx_1^{(t-1)}-\bar{\bx}_{s,1}^{(t-1)}) \\\nonumber
	& \qquad \qquad \qquad \qquad \qquad \qquad \qquad \qquad + \alpha\bs_1^{(t-1)}\|\\\nonumber
	&\leq 2\|\bx_1^{(t-1)}-\bar{\bx}_{s,1}^{(t-1)}\| + \alpha\|\bs_1^{(t-1)}\|\\ \nonumber
	&\leq \alpha\|\bs_1^{(t-1)} - \bg_s(\bx_{1}^{(t-1)})\| + (2+\alpha L_1)\|\bx_1^{(t-1)} \\ 
	&  - \bar{\bx}_{s,1}^{(t-1)}\|  + \alpha L_1\sqrt{M}\|\bar{\bx}_{1}^{(t-1)} - \bx_1^*\| \quad \text{using \eqref{eq:bound_s1}},
\end{align*}
where the second last inequality is because $\|\frac{1}{2}((\bI_M+\bW)\otimes\bI_d) - \bI_{Md}\| \leq \|\frac{1}{2}(\bI_M+\bW)\| + \|\bI_{Md}\| = 2$. 
Using the above inequality in~\eqref{eq:eqn_sys11}, we get
\begin{align}\nonumber
	\|\bs_1^{(t)} - \bg_s(\bx_{1}^{(t)})\| &\leq  (\frac{1+\beta}{2} + \alpha L_1)\|\bs_1^{(t-1)} - \bg_s(\bx_{1}^{(t-1)})\| \\ \nonumber
	&+ L_1(2+\alpha L_1)\|\bx_1^{(t-1)} - \bar{\bx}_{s,1}^{(t-1)}\|  \\\label{eq:eqn_sys1}
	&+ \alpha L_1^2\sqrt{M}\|\bar{\bx}_{1}^{(t-1)} - \bx_1^*\|. 
\end{align}
Writing a system of equations from~\eqref{eq:eqn_sys2}, \eqref{eq:eqn_sys3} and~\eqref{eq:eqn_sys1}, we have the following:
\begin{align}\nonumber
	\begin{bmatrix}
		\|\bs_1^{(t)} - \bg_s(\bx_{1}^{(t)})\| \\
		\|\bx_1^{(t)} - \bar{\bx}_{s,1}^{(t)}\| \\
		\sqrt{M}\|\bar{\bx}_{1}^{(t)} - \bx_1^*\|
	\end{bmatrix} &\leq 
	\begin{bmatrix}
		(\frac{1+\beta}{2} + \alpha L_1) & L_1(2+\alpha L_1) & \alpha L_1^2\\
		\alpha & \frac{1+\beta}{2} & 0\\
		0 & \alpha L_1 & \delta_1
	\end{bmatrix}\\
	&\qquad \times \begin{bmatrix}
		\|\bs_1^{(t-1)} - \bg_s(\bx_{1}^{(t-1)})\| \\
		\|\bx_1^{(t-1)} - \bar{\bx}_{s,1}^{(t-1)}\| \\
		\sqrt{M}\|\bar{\bx}_{1}^{(t-1)} - \bx_1^*\|
	\end{bmatrix},
\end{align}
where $\leq$ implies element-wise inequalities. Let us define $\bP(\alpha) = 	\begin{bmatrix}
		(\frac{1+\beta}{2} + \alpha L_1) & L_1(2+\alpha L_1) & \alpha L_1^2\\
		\alpha & \frac{1+\beta}{2} & 0\\
		0 & \alpha L_1 & \delta_1
	\end{bmatrix}.$
Since $\bP(\alpha)$ has non-negative entries and $\bP^2(\alpha)$ has all positive entries, each entry of $\bP^t(\alpha)$ will be $\cO(\rho(\bP(\alpha))^t)$, where $\rho(\bP(\alpha))$ is the spectral radius of $\bP(\alpha)$. If we choose $\alpha$ such that $\rho(\bP(\alpha))$ is $ < 1$, then that implies $\|\bs_1^{(t)} - \bg_{v,1}^{(t)}\|$, $\|\bx_1^{(t)} - \bar{\bx}_{v,1}^{(t)}\|$ and $\|\bar{\bx}_{1}^{(t)} - \bx_1^*\|$ converge at a linear rate. To find the required condition on $\alpha$, we show in Lemma~\ref{lemma:boundP} in Appendix~\ref{app:boundP} that if $\alpha < \frac{\lambda_1-\lambda_2}{42}(\frac{1-\beta}{9\lambda_1})^2$, the spectral radius of $\bP(\alpha)$ is strictly less than 1. This implies that if $\alpha < \frac{\lambda_1-\lambda_2}{42}(\frac{1-\beta}{9\lambda_1})^2$, then $\|\bar{\bx}_{1}^{(t)} - \bx_1^*\|$, $\|\bx_{i,1}^{(t)} - \bar{\bx}_{1}^{(t)}\|$ and $\|\bs_{i,1}^{(t)} - \bg_{1}^{(t)}\|$ converge at a linear rate to 0. In other words, $\bx_{i,1}^{(t)}$ converges linearly to $\bx_1^* = \pm c_1\bq_1$, where $c_1$ is some constant. 
\end{proof} 

\textbf{\textit{Case II for Induction -- $2\leq k \leq K$}:} \\
We proceed with the proof of convergence for the rest of the eigenvectors through induction.  Assume that $\bx_{i,p}^{(t)}$ converges to $\pm c_p\bq_{p}$ for $p = 1,\ldots, k-1$ linearly, i.e., there exist constants $b_i > 0$ and $\nu_{i} < 1$ such that
    \begin{equation}\label{eq:assumption_edsa}
    \|\sum_{p=1}^{k-1}\Big(\frac{\bx_{i,p}^{(t)}(\bx_{i,p}^{(t)})^T}{\|\bx_{i,p}^{(t)}\|^2} - \bq_p\bq_p^T\Big)\| \leq b_{i}\nu_{i}^t.
    \end{equation}
 In Case I, we proved $\bx_{i,1}^{(t)}$ converges to $\pm c_1\bq_1$ linearly. By induction, we assume $\bx_{i,p}^{(t)}$ converges to $\pm c_p\bq_p$ for $p =1,\ldots,(k-1)$ at a linear rate, which leads to the inequality. We use~\eqref{eq:assumption_edsa} to prove $\bx_{i,k}^{(t)}$ converges to $\pm c_k\bq_k$.
\begin{lemma}\label{lemma:lipshitz_continuousk}
The function $\bh_{i,t}: \R^d \rightarrow \R^d$ with $\bh_{i,t}(\bv) = \bC_i\bv - \frac{(\bv)^T\bC_i\bv}{\|\bv\|^2}\bv - - \sum_{p=1}^{k-1}\frac{\bx_{i,p}^{(t)}(\bx_{i,p}^{(t)})^T}{\|\bx_{i,p}^{(t)}\|^2}\bC_i\bv$ is Lipschitz continuous with constant $L_k = \lambda_1(k+5)$.
\end{lemma}
The proof of this lemma is deferred to Appendix~\ref{app:lemma_lipschitzk}. Using this lemma and the definition of $\bg(\bx_k)$, the following lemma holds true, the proof of which is the same as that of Lemma~\ref{lemma:inequalities}.

\begin{lemma}\label{lemma:inequalitiesk}
	The following inequalities hold with $L_k=\lambda_1(k+5)$ :
	\begin{enumerate}
		\item $\|\bh(\bx_k^{(t)}) - \bh(\bx_k^{(t-1)})\|_2 \leq L_k\|\bx_k^{(t)} - \bx_k^{(t-1)}\|_2$
		\item $\|\bg(\bx_{k}^{(t)}) - \bg(\bx_{k}^{(t-1)})\|_2 \leq \frac{L_k}{\sqrt{M}}\|\bx_k^{(t)} - \bx_k^{(t-1)}\|_2$
		\item $\|\bg(\bx_{k}^{(t)}) - {\bg}(\bar{\bx}_{s,k}^{(t)})\|_2 \leq \frac{L_k}{\sqrt{M}}\|\bx_k^{(t)} - {\bx}_{s,k}^{(t)}\|_2$
	\end{enumerate}
\end{lemma}
\begin{theorem}\label{theorem:theoremk}
Suppose $\alpha < \frac{\min_{k=1,\ldots,K} (\lambda_k-\lambda_{k+1})}{(K+5)(K+6)}(\frac{1-\beta}{9\lambda_1})^2$ where $\lambda_k, \lambda_{k+1}$ are the $k^{th}$ and $(k+1)^{th}$ largest eigenvalues of $\bC$, $\beta = \max \{|\lambda_2(\bW)|,|\lambda_M(\bW)|\}$, $\bq_k^T\bar{\bx}_{k}^{(0)} \neq 0$, and the graph underlying the network is connected. Then the estimate $\bx_{i,k}^{(t)}$ from FAST-PCA converges to the eigenvector $\pm c_k\bq_k$ corresponding to the largest eigenvalue $\lambda_k$ of $\bC$ at each node $i = 1,\ldots,M$ at a linear rate.
\end{theorem}
\begin{proof} Using the definitions of $\bx_k^{(t)}$, $\bs_k^{(t)}$, $\bg_s(\bx_k^{(t)})$, $\bh(\bx_k^{(t)})$ and same algebraic manipulations as in Theorem~\ref{theorem:theorem1}, we get
\begin{align}\nonumber
	\|\bs_k^{(t)}- \bg_s(\bx_{k}^{(t)})\| &\leq \frac{1+\beta}{2}\|\bs_k^{(t-1)} - \bg_s(\bx_{k}^{(t-1)})\| \\ \label{eq:eqn_sys11k}
	& \qquad \qquad + L_k\|\bx_k^{(t)} - \bx_k^{(t-1)}\|
\end{align}
and,
\begin{equation} \label{eq:eqn_sys2k}
	\|\bx_k^{(t)} - \bar{\bx}_{s,k}^{(t)}\| \leq \frac{1+\beta}{2}\|\bx_k^{(t-1)} - \bar{\bx}_{s,k}^{(t-1)}\| + \alpha\|\bs_k^{(t-1)} - \bg_s(\bx_{k}^{(t-1)})\|.
\end{equation}
Now, we bound $\|\bar{\bx}_{k}^{(t)} - \bx_k^*\|$. We know
\begin{align*}
	{\bg}(\bar{\bx}_{s,k}^{(t)}) &= \frac{1}{M}\Big(\bC\bar{\bx}_{k}^{(t)} - \frac{(\bar{\bx}_{k}^{(t)})^T\bC\bar{\bx}_{k}^{(t)}}{\|\bar{\bx}_{k}^{(t)}\|^2}\bar{\bx}_{k}^{(t)} - \sum_{i=1}^{M} \sum_{p=1}^{k-1}\frac{\bx_{i,p}^{(t)}(\bx_{i,p}^{(t)})^T}{\|\bx_{i,p}^{(t)}\|^2}\bC_i\bar{\bx}_{k}^{(t)}\Big)\\
	 &= \frac{1}{M}\Big(\bC\bar{\bx}_{k}^{(t)} - \frac{(\bar{\bx}_{k}^{(t)})^T\bC\bar{\bx}_{k}^{(t)}}{\|\bar{\bx}_{k}^{(t)}\|^2}\bar{\bx}_{k}^{(t)} - \sum_{i=1}^{M} \sum_{p=1}^{k-1}\bq_{p}\bq_{p}^T\bC_i\bar{\bx}_{k}^{(t)}\Big)  \\
	 & \qquad \qquad -\frac{1}{M}\sum_{i=1}^{M} \sum_{p=1}^{k-1}\Big(\frac{\bx_{i,p}^{(t)}(\bx_{i,p}^{(t)})^T}{\|\bx_{i,p}^{(t)}\|^2} - \bq_p\bq_p^T\Big)\bC_i\bar{\bx}_{k}^{(t)}\\
	  &= \frac{1}{M}\Big(\bC\bar{\bx}_{k}^{(t)} - \frac{(\bar{\bx}_{k}^{(t)})^T\bC\bar{\bx}_{k}^{(t)}}{\|\bar{\bx}_{k}^{(t)}\|^2}\bar{\bx}_{k}^{(t)} -  \sum_{p=1}^{k-1}\bq_{p}\bq_{p}^T\bC\bar{\bx}_{k}^{(t)}\Big) \\
	  &\qquad \qquad -\frac{1}{M}\sum_{i=1}^{M} \sum_{p=1}^{k-1}\Big(\frac{\bx_{i,p}^{(t)}(\bx_{i,p}^{(t)})^T}{\|\bx_{i,p}^{(t)}\|^2} - \bq_p\bq_p^T\Big)\bC_i\bar{\bx}_{k}^{(t)}\\
	  &= {\bg}^{'}(\bar{\bx}_{s,k}^{(t)}) - {\bff}(\bar{\bx}_{s,k}^{(t)})
\end{align*}
where, ${\bg}^{'}(\bar{\bx}_{s,k}^{(t)}) = \frac{1}{M}(\bC\bar{\bx}_{k}^{(t)} - \frac{(\bar{\bx}_{k}^{(t)})^T\bC\bar{\bx}_{k}^{(t)}}{\|\bar{\bx}_{k}^{(t)}\|^2}\bar{\bx}_{k}^{(t)} -  \sum_{p=1}^{k-1}\bq_{p}\bq_{p}^T\bC\bar{\bx}_{k}^{(t)})$ and ${\bff}(\bar{\bx}_{s,k}^{(t)}) = \frac{1}{M}\sum_{i=1}^{M} \sum_{p=1}^{k-1}(\frac{\bx_{i,p}^{(t)}(\bx_{i,p}^{(t)})^T}{\|\bx_{i,p}^{(t)}\|^2} - \bq_p\bq_p^T)\bC_i\bar{\bx}_{k}^{(t)}$. From~\eqref{eq:edsa_xavg}, we have
\begin{align*}
	\bar{\bx}_{k}^{(t)} &= \bar{\bx}_{k}^{(t-1)} + \alpha\bar{\bs}_{k}^{(t-1)} = \bar{\bx}_{k}^{(t-1)} + \alpha\bg(\bx_{k}^{(t-1)}) \\
	&= \bar{\bx}_{k}^{(t-1)} + \alpha{\bg}(\bar{\bx}_{s,k}^{(t-1)}) + \alpha(\bg(\bx_{k}^{(t-1)}) - {\bg}(\bar{\bx}_{s,k}^{(t-1)}))\\
	&= \bar{\bx}_{k}^{(t-1)} + \frac{\alpha}{M}(\bC\bar{\bx}_{k}^{(t-1)} - \frac{(\bar{\bx}_{k}^{(t-1)})^T\bC\bar{\bx}_{k}^{(t-1)}}{\|\bar{\bx}_{k}^{(t-1)}\|^2}\bar{\bx}_{k}^{(t-1)} - \\
	&\sum_{p=1}^{k-1}\bq_{p}\bq_{p}^T\bC\bar{\bx}_{k}^{(t-1)})  - \alpha{\bff}(\bar{\bx}_{s,k}^{(t-1)}) + \alpha(\bg(\bx_{k}^{(t-1)}) - {\bg}(\bar{\bx}_{s,k}^{(t-1)})).
\end{align*}
Using~\eqref{eq:centralized_adsa_kras} and~\eqref{eq:convergence_kras}, we know that an iterate of the form
\begin{align*}
    \bar{\bx}_{k}^{(t)} &= \bar{\bx}_{k}^{(t-1)} +  \alpha(\bC\bar{\bx}_{k}^{(t-1)} - \frac{(\bar{\bx}_{k}^{(t-1)})^T\bC\bar{\bx}_{k}^{(t-1)}}{\|\bar{\bx}_{k}^{(t-1)}\|^2}\bar{\bx}_{k}^{(t-1)} \\
    & \qquad\qquad\qquad\qquad\qquad\qquad - \sum_{p=1}^{k-1}\bq_{p}\bq_{p}^T\bC\bar{\bx}_{k}^{(t-1)})
\end{align*}
converges linearly for $\alpha < \frac{1}{\lambda_1}$ and $\bq_k^T\bar{\bx}_{k}^{(0)} \neq 0$ as
\begin{align*}
    \|\bar{\bx}_{k}^{(t)} - \bx_k^*\| \leq \delta_k\|\bar{\bx}_{k}^{(t-1)} - \bx_k^*\|,
\end{align*}
where $\bx_k^* = \pm c_k\bq_k$ and $\delta_k = \frac{1+\alpha\lambda_{k+1}}{1+\alpha\lambda_k}$. Thus,
using~\eqref{eq:centralized_adsa_kras} and \eqref{eq:convergence_kras}, we know
\begin{align} \nonumber
	\|\bar{\bx}_{k}^{(t)} - \bx_k^*\| &\leq \delta_k\| \bar{\bx}_{k}^{(t-1)} - \bx_k^*\| +\alpha \|\bg(\bx_{k}^{(t-1)}) - {\bg}(\bar{\bx}_{s,k}^{(t-1)})\|  + \\
	& \qquad\qquad\qquad \qquad\qquad\qquad \alpha\|{\bff}(\bar{\bx}_{s,k}^{(t-1)})\|\\\nonumber
	&\leq \delta_k\|\bar{\bx}_{k}^{(t-1)} - \bx_k^*\| + \alpha\frac{L_k}{\sqrt{M}} \|\bx_k^{(t-1)} - \bar{\bx}_{s,k}^{(t-1)}\| + \\ \label{eq:eqn_sys3k}
	& \qquad\qquad\qquad  \qquad\qquad\quad  \alpha\|{\bff}(\bar{\bx}_{s,k}^{(t-1)})\|.
\end{align}
Now, we will bound $\|\bx_k^{(t)} - \bx_k^{(t-1)}\|$. Since ${\bg}^{'}(\bar{\bx}_{s,k}^{(t)}) = \frac{1}{M}(\bC\bar{\bx}_{k}^{(t)} - \frac{(\bar{\bx}_{k}^{(t)})^T\bC\bar{\bx}_{k}^{(t)}}{\|\bar{\bx}_{k}^{(t)}\|^2}\bar{\bx}_{k}^{(t)} -  \sum_{p=1}^{k-1}\bq_{p}\bq_{p}^T\bC\bar{\bx}_{k}^{(t)})$,  we have ${\bg}^{'}(\bx_{s,k}^{*}) = \frac{1}{M}(c_k\bC\bq_{k} - \frac{c_k\bq_{k}^T\bC c_k\bq_k}{c_k^2\|\bq_{k}\|^2}\bq_{k} -  \sum_{p=1}^{k-1}\bq_{p}\bq_{p}^T\bC c_k\bq_{k}) = 0.$ Hence, 
\begin{align*}
	\|{\bg}_s^{'}(\bar{\bx}_{s,k}^{(t-1)})\| &= \sqrt{M}\|{\bg}^{'}(\bar{\bx}_{s,k}^{(t-1)})\| \\
	&= \sqrt{M}\|{\bg}^{'}(\bar{\bx}_{s,k}^{(t-1)}) - {\bg}^{'}({\bx}_{s,k}^*)\| \\
	&\leq L_k\sqrt{M}\|\bar{\bx}_{k}^{(t-1)} - \bx_k^*\|.
\end{align*}
Using the above inequality and Lemma~\ref{lemma:inequalitiesk}, we get
\begin{align}\nonumber
	\|\bs_k^{(t-1)}\| &= \|\bs_k^{(t-1)} - \bg_s(\bx_{k}^{(t-1)}) + \bg_s(\bx_{k}^{(t-1)})  - {\bg}_s(\bar{\bx}_{s,k}^{(t-1)}) \\\nonumber
	& \qquad\qquad\qquad + {\bg}_s^{'}(\bar{\bx}_{s,k}^{(t-1)}) - {\bff}_s(\bar{\bx}_{s,k}^{(t-1)})\|\\ \nonumber
	&\leq \|\bs_k^{(t-1)} - \bg_s(\bx_{k}^{(t-1)})\| + \|\bg_s(\bx_{k}^{(t-1)})  - {\bg}_s(\bar{\bx}_{s,k}^{(t-1)}) \| \\\nonumber
	& \qquad\qquad\qquad + \|{\bg}_s^{'}(\bar{\bx}_{s,k}^{(t-1)})\| + \|{\bff}_s(\bar{\bx}_{s,k}^{(t-1)})\|\\\nonumber
	&\leq \|\bs_k^{(t-1)} - \bg_s(\bx_{k}^{(t-1)})\| + L_k\|\bx_k^{(t-1)} - \bar{\bx}_{s,k}^{(t-1)}\| \\\label{eq:bound_sk}
	& +L_k\sqrt{M}\|\bar{\bx}_{k}^{(t-1)} - \bx_k^*\| + \sqrt{M}\|{\bff}(\bar{\bx}_{s,k}^{(t-1)})\|.
\end{align}
Thus,
\begin{align}\nonumber
	\|\bx_k^{(t)} - \bx_k^{(t-1)}\| &= \|(\bW\otimes\bI)\bx_k^{(t-1)} - \bx_k^{(t-1)} + \alpha\bs_k^{(t-1)}\| \\\nonumber
	&=  \|(\bW\otimes\bI - \bI)(\bx_k^{(t-1)}-\bar{\bx}_{s,k}^{(t-1)}) + \alpha\bs_k^{(t-1)}\|\\\nonumber
	&\leq 2\|\bx_k^{(t-1)}-\bar{\bx}_{s,k}^{(t-1)}\| + \alpha\|\bs_k^{(t-1)}\|\\ \label{eq:eqn_sys12k}
	&\leq \alpha\|\bs_k^{(t-1)} - \bg_s(\bx_{k}^{(t-1)})\| + (2+\alpha L_k)\times \\\nonumber
	&\|\bx_k^{(t-1)} - \bar{\bx}_{s,k}^{(t-1)}\| + \alpha L_k\sqrt{M}\|\bar{\bx}_{k}^{(t-1)} - \bx_k^*\|  \\
	& + \alpha \sqrt{M}\|{\bff}(\bar{\bx}_{s,k}^{(t-1)})\|\quad \text{using \eqref{eq:bound_sk}}
\end{align}
Using the above inequality in~\eqref{eq:eqn_sys11k}, we get
\begin{align}\nonumber
	&\|\bs_k^{(t)} - \bg_s(\bx_{k}^{(t)})\| \leq  (\frac{1+\beta}{2}+ \alpha L_k)\|\bs_k^{(t-1)} - \bg_s(\bx_{k}^{(t-1)})\| \\\nonumber
	&\qquad\qquad\qquad+ L_k(2+\alpha L_k)\|\bx_k^{(t-1)} - \bar{\bx}_{s,k}^{(t-1)}\|  + \\ \label{eq:eqn_sys1k}
	& \alpha L_k^2\sqrt{M}\|\bar{\bx}_{k}^{(t-1)} - \bx_k^*\| + \alpha L_k \sqrt{M}\|{\bff}(\bar{\bx}_{s,k}^{(t-1)})\|.
\end{align}
Writing a system of equations from~\eqref{eq:eqn_sys1k}, \eqref{eq:eqn_sys2k} and \eqref{eq:eqn_sys3k}, we have the following:
\begin{align}\nonumber
	&\begin{bmatrix}
		\|\bs_k^{(t)} - \bg_s(\bx_{k}^{(t)})\| \\
		\|\bx_k^{(t)} - \bar{\bx}_{s,k}^{(t)}\| \\
		\sqrt{M}\|\bar{\bx}_{k}^{(t)} - \bx_k^*\|
	\end{bmatrix} \leq
	\begin{bmatrix}
		(\frac{1+\beta}{2} + \alpha L_k) & L_k(2+\alpha L_k) & \alpha L_k^2\\
		\alpha & \frac{1+\beta}{2} & 0\\
		0 & \alpha L_k & \delta_k
	\end{bmatrix}\\
	&
	\times \begin{bmatrix}
		\|\bs_k^{(t-1)} - \bg_s(\bx_{k}^{(t-1)})\| \\
		\|\bx_k^{(t-1)} - \bar{\bx}_{s,k}^{(t-1)}\| \\
		\sqrt{M}\|\bar{\bx}_{k}^{(t-1)} - \bx_k^*\|
	\end{bmatrix} +
\|{\bff}(\bar{\bx}_{s,k}^{(t-1)})\|\begin{bmatrix}
 \alpha L_k \sqrt{M}  \\
0\\
	\alpha\sqrt{M}
\end{bmatrix}.
\end{align}
Let us define $\bP_k(\alpha) = \begin{bmatrix}
		(\frac{1+\beta}{2} + \alpha L_k) & L_k(2+\alpha L_k) & \alpha L_k^2\\
		\alpha & \frac{1+\beta}{2} & 0\\
		0 & \alpha L_k & \rho_k
	\end{bmatrix}.$
Since $\bP_k(\alpha)$ has non-negative entries and $\bP_k^2(\alpha)$ has all positive entries, each entry of $\bP_k^t(\alpha)$ will be $\cO(\rho(\bP_k(\alpha))^t)$, where $\rho(\bP_k(\alpha))$ is the spectral radius of $\bP_k(\alpha)$. From Lemma~\ref{lemma:boundP} in Appendix~\ref{app:boundP} we know if we choose $\alpha < \frac{\lambda_k-\lambda_{k+1}}{(k+5)(k+6)}(\frac{1-\beta}{9\lambda_1})^2$, then $\rho(\bP_k(\alpha))<1$.
Also, we know ${\bff}(\bar{\bx}_{s,k}^{(t)}) = \frac{1}{M}\sum_{i=1}^{M} \sum_{p=1}^{k-1}(\frac{\bx_{i,p}^{(t)}(\bx_{i,p}^{(t)})^T}{\|\bx_{i,p}^{(t)}\|^2} - \bq_p\bq_p^T)\bC_i\bar{\bx}_{k}^{(t)}$. Thus,
\begin{align*}
    \|{\bff}(\bar{\bx}_{s,k}^{(t)})\| &= \|\frac{1}{M}\sum_{i=1}^{M} \sum_{p=1}^{k-1}(\frac{\bx_{i,p}^{(t-1)}(\bx_{i,p}^{(t-1)})^T}{\|\bx_{i,p}^{(t-1)}\|^2} - \bq_p\bq_p^T)\bC_i\bar{\bx}_{k}^{(t-1)}\| \\
    &\leq \frac{1}{M}\sum_{i=1}^{M} \|\sum_{p=1}^{k-1}(\frac{\bx_{i,p}^{(t-1)}(\bx_{i,p}^{(t-1)})^T}{\|\bx_{i,p}^{(t-1)}\|^2} - \bq_p\bq_p^T)\|\|\bC_i\|\|\bar{\bx}_{k}^{(t-1)}\|.
\end{align*}
From~\eqref{eq:assumption_edsa}, we know $\|\sum_{p=1}^{k-1}(\frac{\bx_{i,p}^{(t-1)}(\bx_{i,p}^{(t-1)})^T}{\|\bx_{i,p}^{(t)}\|^2} - \bq_p\bq_p^T)\| \leq b_{i}\nu_{i}^t$. Let $b = (\max_i b_i)\lambda_1\|\bar{\bx}_{k}^{(t-1)}\| > 0$ and $\nu = \max_i \nu_i < 1$.  
Thus the system of equations becomes
\begin{align}\nonumber
	\begin{bmatrix}
		\|\bs_k^{(t)} - \bg_s(\bx_{k}^{(t)})\| \\
		\|\bx_k^{(t)} - \bar{\bx}_{s,k}^{(t)}\| \\
		\sqrt{M}\|\bar{\bx}_{k}^{(t)} - \bx_k^*\|
	\end{bmatrix} &\leq
	\rho(\bP_k(\alpha))^t
		\begin{bmatrix}
		\|\bs_k^{(0)} - \bg_s(\bx_{k}^{(0)})\| \\
		\|\bx_k^{(0)} - \bar{\bx}_{s,k}^{(0)}\| \\
		\sqrt{M}\|\bar{\bx}_{k}^{(0)} - \bx_k^*\|
	\end{bmatrix}  \\
	& + b\nu^t
\begin{bmatrix}
 \alpha L_k \sqrt{M}  \\
0\\
	\alpha\sqrt{M}
\end{bmatrix}.
\end{align}
This implies that if $\alpha < \frac{\lambda_k-\lambda_{k+1}}{(k+5)(k+6)}(\frac{1-\beta}{9\lambda_1})^2$, then $\|\bar{\bx}_{k}^{(t)} - \bx_k^*\|$, $\|\bx_{i,k}^{(t)} - \bar{\bx}_{k}^{(t)}\|$ and $\|\bs_{i,k}^{(t)} - \bg_{k}^{(t)}\|$ converge at a linear rate to 0. In other words, $\bx_{i,k}^{(t)}$ converges linearly to $\bx_k^* = \pm c_k\bq_k$, where $c_k$ is some constant.
\end{proof} 
From Theorem~\ref{theorem:theorem1} and Theorem~\ref{theorem:theoremk}, we can see that if $\alpha < \frac{\min_k(\lambda_k-\lambda_{k+1})}{(K+5)(K+6)}(\frac{1-\beta}{9\lambda_1})^2$, where $\lambda_1$ is the largest eigenvalue of $\bC$, $K$ is the number of eigenvectors to be estimated and $\beta$ is the absolute value of the second-largest eigenvalue of the weight matrix $\bW$, then the estimates $\bx_{i,k}^{(t)}$ of the $k^{th}$ eigenvector for $k=1,\ldots,K$ at $i^{th}$ node, $i=1,\ldots,M$ converge at a linear rate to a multiple of the eigenvector $\bq_k$ of $\bC$ i.e., $\pm c_k \bq_k$. It is clear from the condition on $\alpha$ that with larger eigengap $(\lambda_k - \lambda_{k+1})$, a larger range of step size is possible which directly affects the rate of convergence. Also, as the connectivity in the network increases, $\beta$ decreases which again increases the range of $\alpha$ thus increasing the rate of convergence. The final step of the algorithm normalizes the estimated eigenvectors to ensure that they are orthonormal. 

\section{Experimental Results}\label{sec:results}
In this section, we demonstrate the efficacy of our proposed FAST-PCA algorithm through experiments on synthetic as well as real-world data. We compare the performance of our algorithm with existing algorithms of (centralized) orthogonal iteration (OI), (centralized) sequential power method (SeqPM), distributed sequential power method (SeqDistPM), distributed orthogonal iteration algorithms (S-DOT, SA-DOT)~\cite{xiang.gang.bajwa.2021}, an orthogonal iteration+gradient tracking based method DeEPCA~\cite{deepca} and our previously proposed distributed Sanger's algorithm (DSA)~\cite{gang.bajwa.2021}. In the case of OI and SeqPM, we assume that all the samples are available at a single location and, for the estimation of $K$ dominant eigenvectors of $\bC$, SeqPM performs power method $K$ times sequentially starting from the most dominant eigenvector. SeqDistPM is the distributed version of SeqPM, which uses an explicit consensus loop with a fixed number $T_c$ of consensus iterations per iteration of the power iteration~\cite{cksvd.allerton.2013, cksvd}, whereas S-DOT and SA-DOT are distributed versions of OI using fixed and increasing number of consensus iterations per orthogonal iteration. The DSA is a distributed generalized Hebbian algorithm that converges linearly to a neighborhood of the true eigenvectors of the global covariance matrix. Assuming that the cost of communicating $\R^{d\times K}$ matrices across the network in one (outer loop) iteration is one unit, the x-axes of all the plots indicate the total communication cost, i.e., total inner and outer loop communications. In the algorithms with one time scale, this is the same as the number of total outer loop iterations (since inner iterations = 0). The y-axes of the plots express the average angle between the estimated eigenvectors $\bx_{i,k}^{(t)}$ and the true eigenvectors $\pm\bq_k$ across all the $M$ nodes in the network given by
\begin{align}
        \cE = \frac{1}{MK}\sum_{i=1}^{M}\sum_{k=1}^{K}\bigg(1 - \bigg(\frac{\bx_{i,k}^T\bq_k}{\|\bx_{i,k}\|}\bigg)^2\bigg).
\end{align}
\subsection{Synthetic Data}
We study the effects of factors like eigengap and distinct/repeated eigenvalues on the performance of our algorithm in comparison to various other existing PCA and distributed PCA algorithms. To that end, we generate Erdos-Renyi graphs ($p=0.5$) and cyclic graphs to simulate the distributed setup with $M=20$ nodes. We also generate synthetic data with different eigengaps of $\Delta_K = \frac{\lambda_{K+1}}{\lambda_K} \in \{0.8, 0.97\}$. The data is generated such that each node has 5000 i.i.d samples, i.e. $N_i = 5000$ with $d=20$ drawn from a multivariate Gaussian distribution with zero mean and fixed covariance matrix $\bSigma$. The number of eigenvectors to be estimated is set to $K =5$. For SeqPM, SeqDistPM and S-DOT, the number of consensus iterations per outer loop iteration is $T_c = 50$ and the number of maximum consensus iterations in the case of SA-DOT is set to $50$ as well. For the Erdos-Renyi topology, we use a step size of $\alpha=0.7$ for our algorithm and for cyclic graph, we use $\alpha=0.1$ picked by trial and error. The results reported are an average of $10$ Monte-Carlo simulations.

Figure~\ref{fig:fast_comp1} compares the performance of our proposed FAST-PCA algorithm with centralized OI, SeqPM, SeqDistPM, S-DOT, SA-DOT, DeEPCA and DSA when the subspace eigenvalues $\lambda_1,\ldots, \lambda_K$ are distinct, i.e. $\lambda_1>\lambda_2>\ldots, >\lambda_K$. It is clear that our algorithm significantly outperforms SeqPM and SeqDistPM since estimating one eigenvector at a time slows down the convergence of these methods. Also, the requirement of an explicit consensus loop implies the communication cost of these methods is high as indicated by the plots. Even though S-DOT and SA-DOT estimate the whole subspace (but not necessarily the eigenvectors) simultaneously, an explicit consensus loop makes those relatively slow as well. As expected, since DSA converges only to a neighborhood of the true solutions, our new proposed algorithm outperforms it. The performance of FAST-PCA is similar to DeEPCA in this case, although DeEPCA requires explicit QR normalization after every iteration whereas FAST-PCA requires no explicit normalization. This normalization step in DeEPCA requires an additional $\cO(K^2d)$ computations per iterations. It is desired from any distributed algorithm to perform similar to their centralized counterparts and it is clear from the figures that our algorithm FAST-PCA performs very similar to centralized OI.

Figure~\ref{fig:fast_comp2} shows a similar performance comparison when the subspace eigenvalues are very close to each other, i.e. $\lambda_1\approx\lambda_2\approx\ldots\approx\lambda_K$. The Gaussian distribution generated in this case has covariance matrix $\bSigma$ with equal subspace eigenvalues but due to the finite number of samples, the eigenvalues of $\bC$ are not exactly equal albeit almost equal. It is evident that the performance of every algorithm significantly deprecates in this scenario. Nonetheless, in this case FAST-PCA outperforms all other algorithms including DeEPCA, while still being close to centralized OI in terms of performance.
\begin{figure}
	\centering
	\begin{subfigure}{.23\textwidth}
		\centering
		\includegraphics[width=\linewidth]{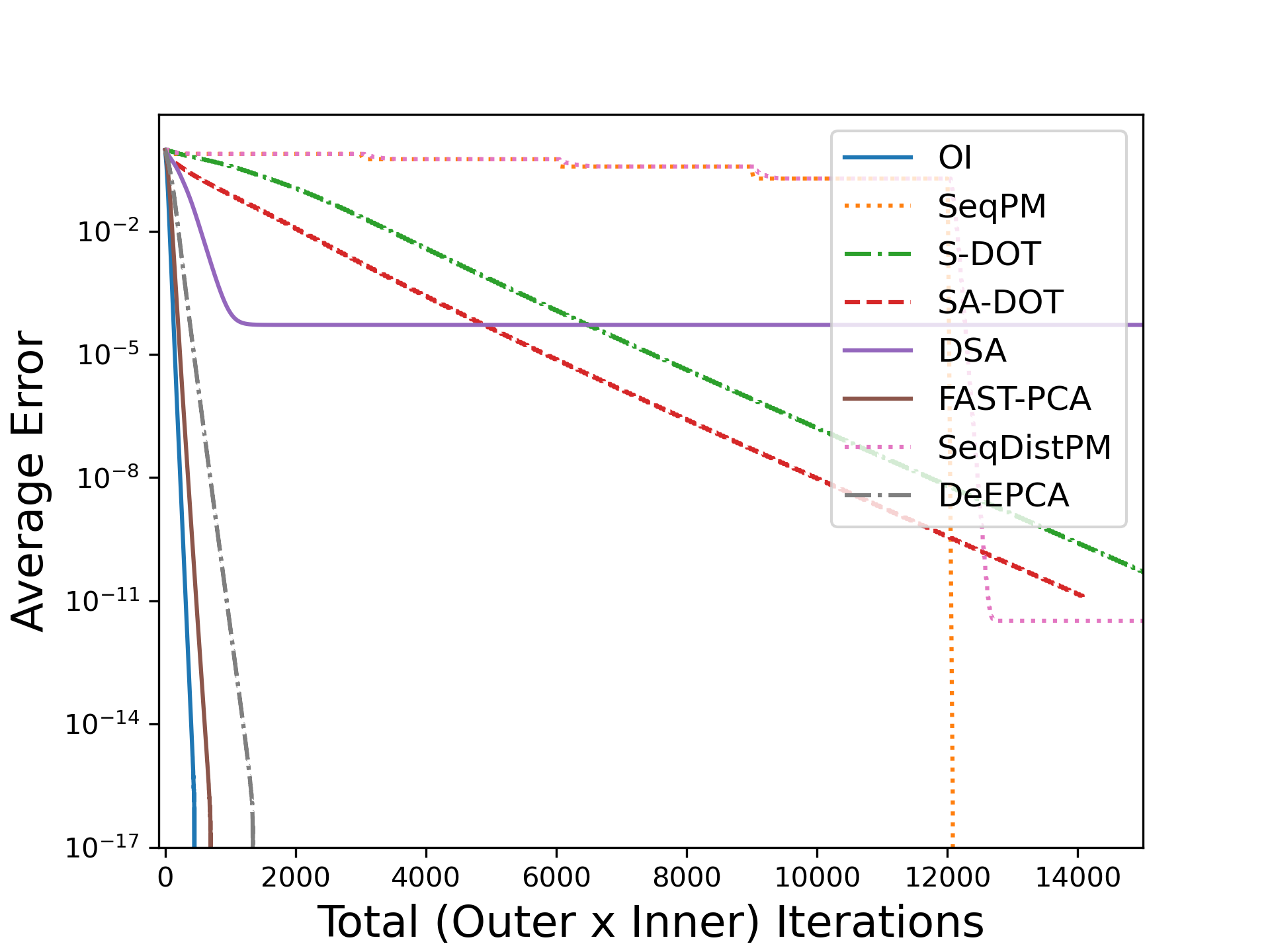}
		\caption{Erdos-Renyi, $\Delta_K = 0.8$}
		\label{fig:a1}
	\end{subfigure}
	\hfil
	\begin{subfigure}{.23\textwidth}
		\centering
		\includegraphics[width=\linewidth]{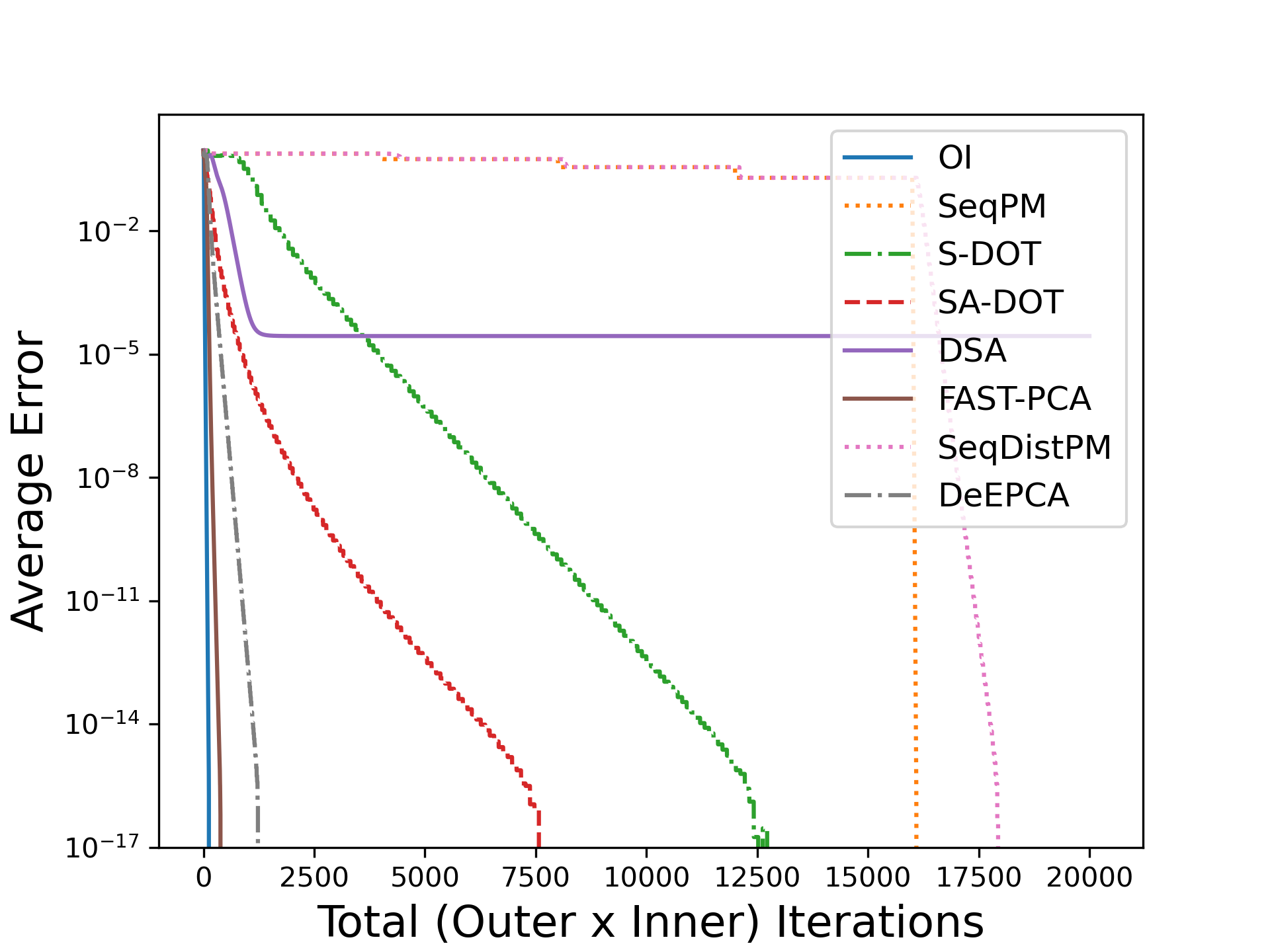}
		\caption{Erdos-Renyi, $\Delta_K = 0.97$}
		\label{fig:b1}
	\end{subfigure}
	\hfil
	\begin{subfigure}{.23\textwidth}
		\centering
		\includegraphics[width=\linewidth]{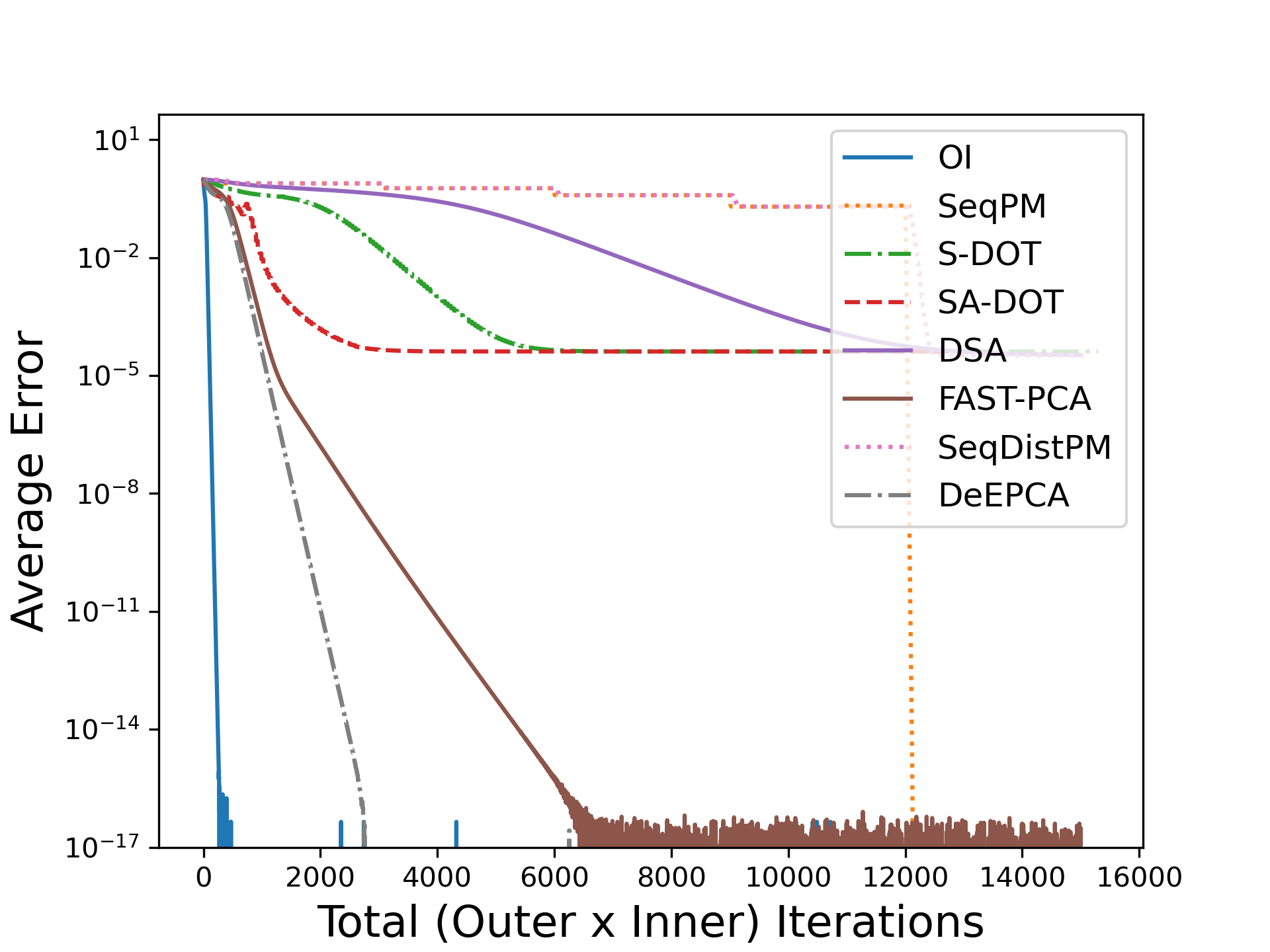}
		\caption{Cycle, $\Delta_K = 0.8$}
		\label{fig:c1}
	\end{subfigure}
	\hfil
	\begin{subfigure}{.23\textwidth}
		\centering
		\includegraphics[width=\linewidth]{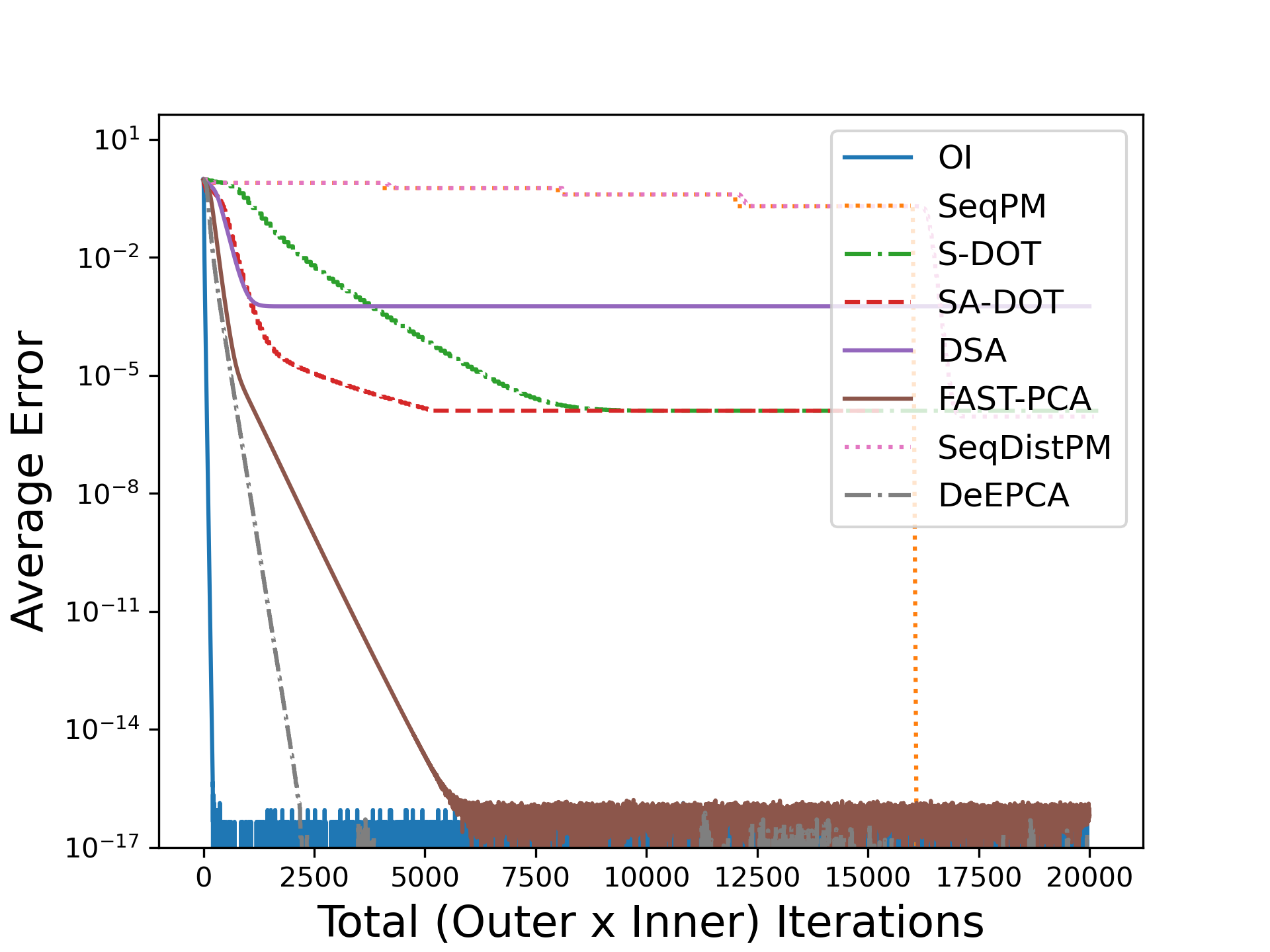}
		\caption{Cycle, $\Delta_K = 0.97$}
		\label{fig:d1}
	\end{subfigure}
	\caption{Performance comparison of FAST-PCA with various algorithms for two different eigengaps and two graph topologies. Here, the top $K+1$ eigenvalues of $\bC$ are distinct.}
	\label{fig:fast_comp1}
\end{figure}

\begin{figure}
	\centering
	\begin{subfigure}{.23\textwidth}
		\centering
		\includegraphics[width=\linewidth]{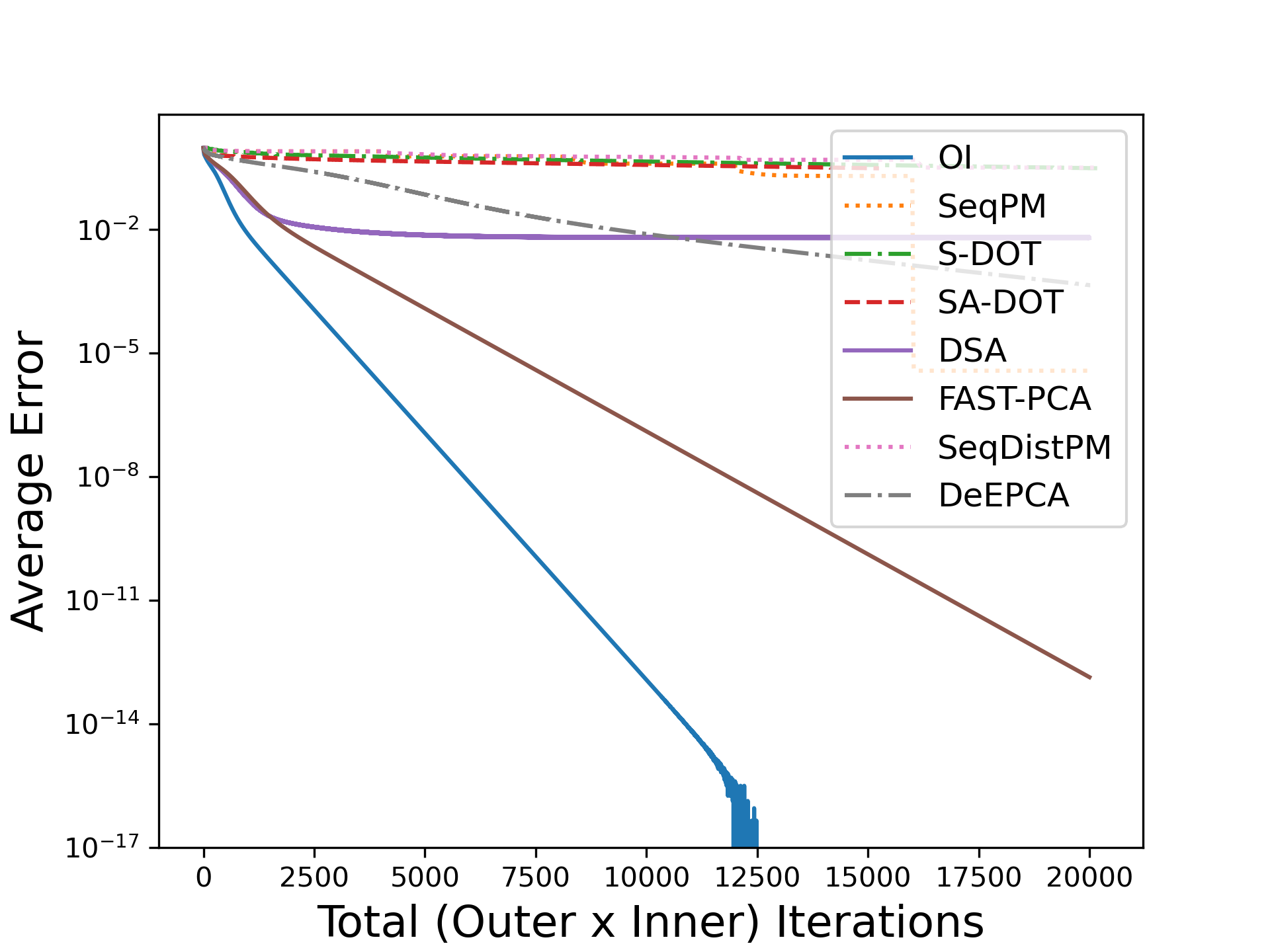}
		\caption{Erdos-Renyi, $\Delta_K = 0.8$}
		\label{fig:a2}
	\end{subfigure}
	\hfil
	\begin{subfigure}{.23\textwidth}
		\centering
		\includegraphics[width=\linewidth]{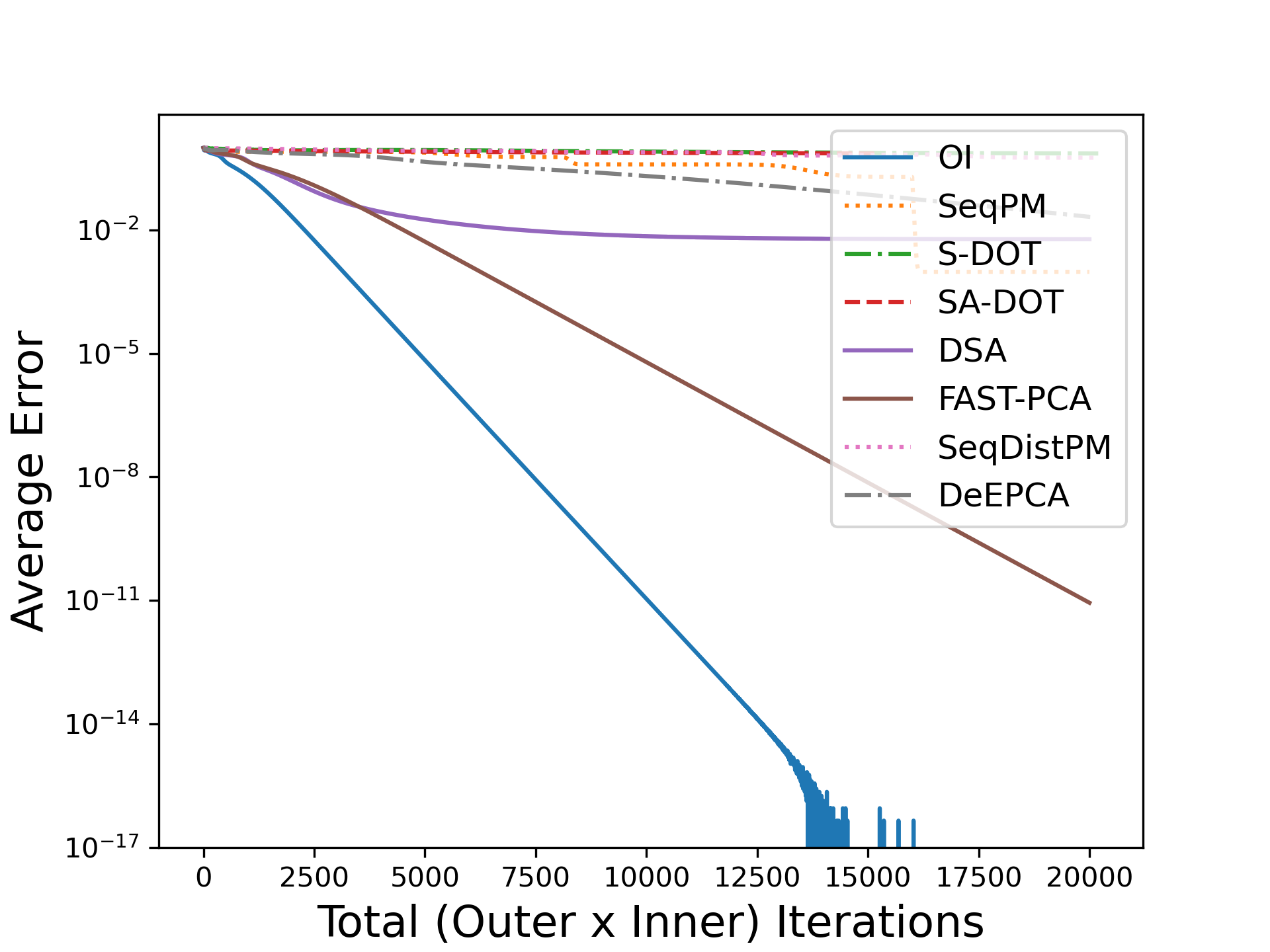}
		\caption{Erdos-Renyi, $\Delta_K = 0.97$}
		\label{fig:b2}
	\end{subfigure}
	\hfil
	\begin{subfigure}{.23\textwidth}
		\centering
		\includegraphics[width=\linewidth]{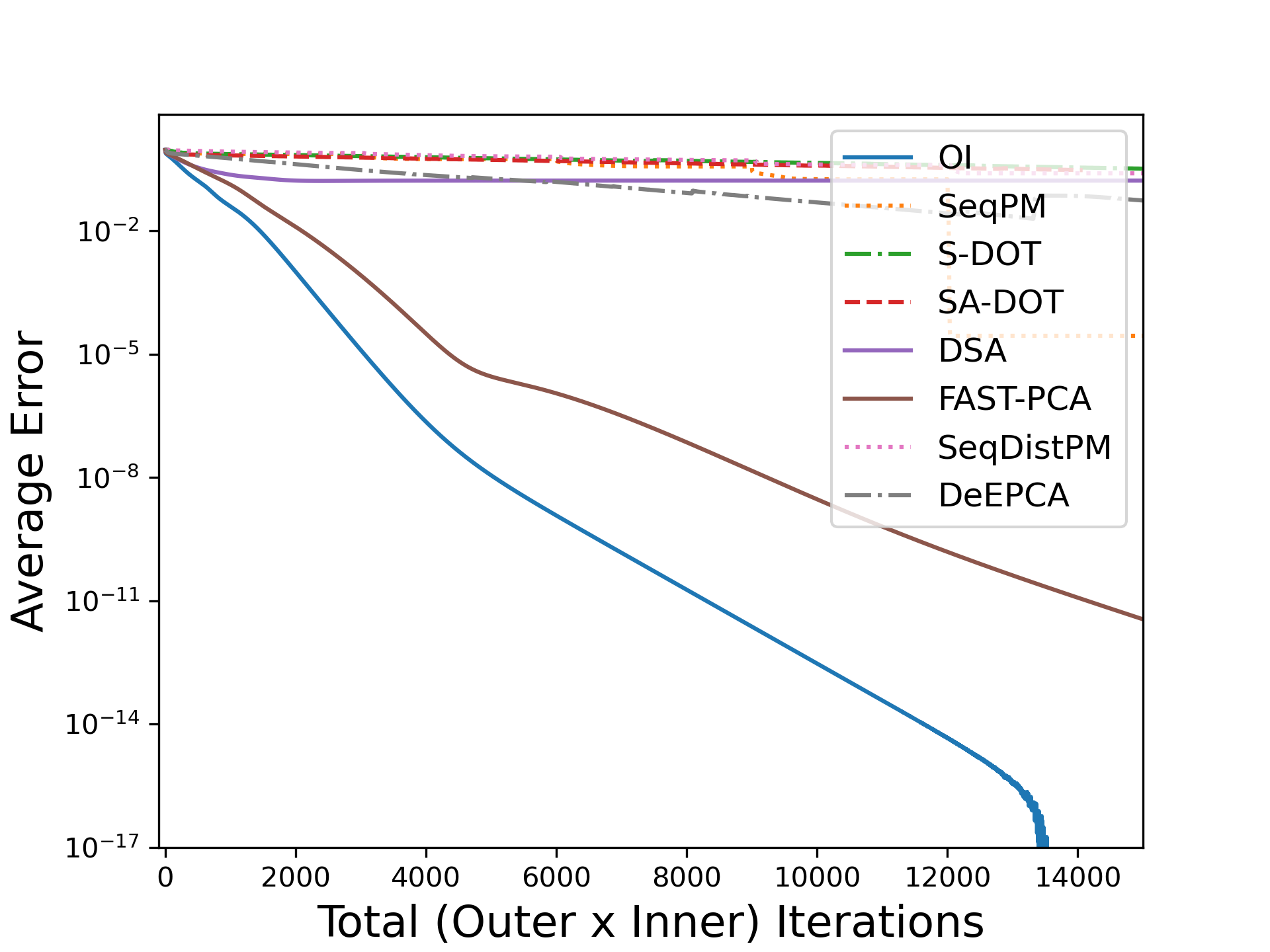}
		\caption{Cycle, $\Delta_K = 0.8$}
		\label{fig:c2}
	\end{subfigure}
	\hfil
	\begin{subfigure}{.23\textwidth}
		\centering
		\includegraphics[width=\linewidth]{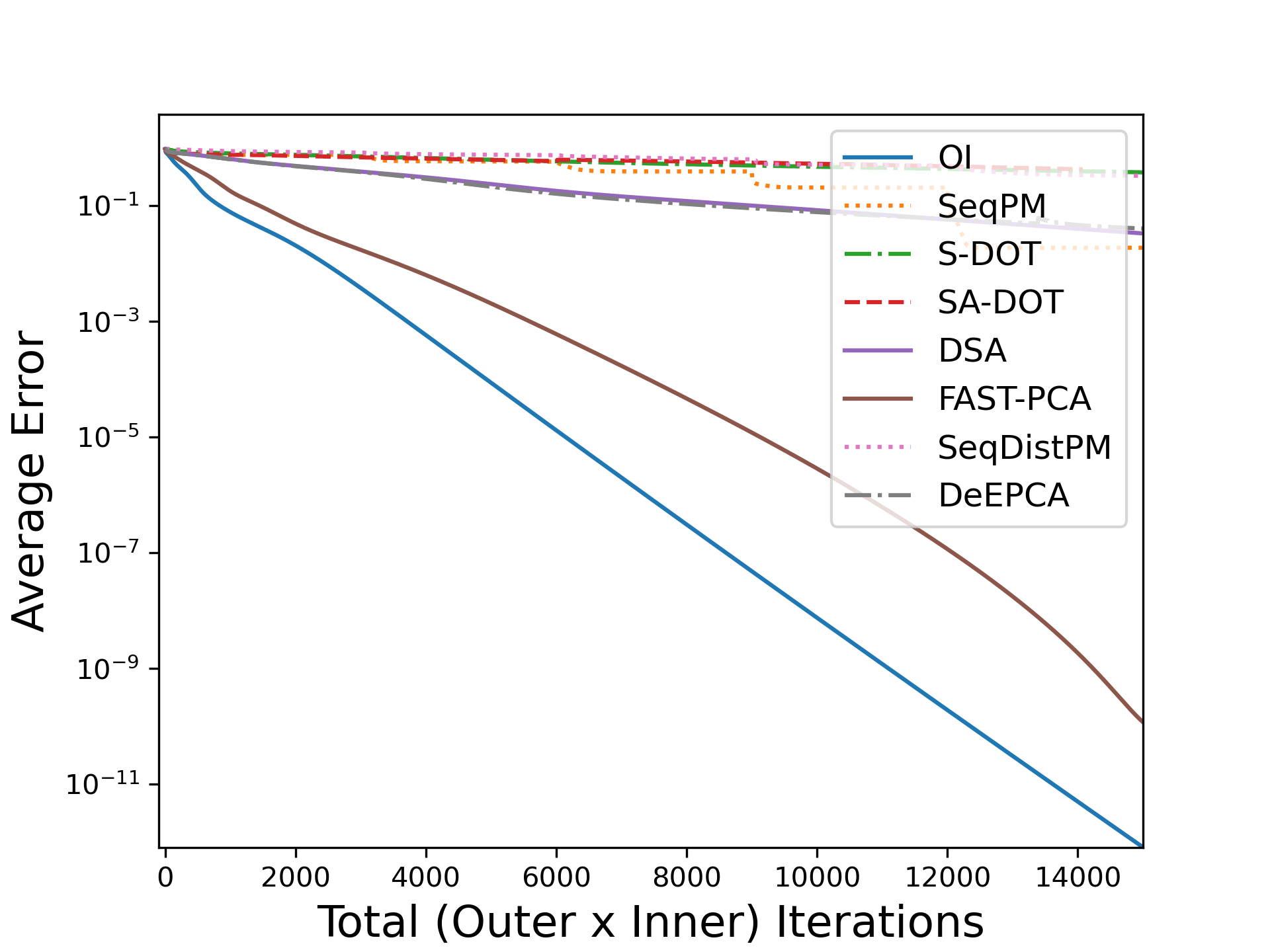}
		\caption{Cycle, $\Delta_K = 0.97$}
		\label{fig:d2}
	\end{subfigure}
	\caption{Performance comparison of FAST-PCA with various algorithms for two different eigengaps and two graph topologies in the case of first $K$ (almost) equal eigenvalues.}
	\label{fig:fast_comp2}
\end{figure}

\subsection{Real-World Data}
We also provide some results for the real-world datasets of MNIST~\cite{mnist.2010} and CIFAR10~\cite{cifar.2009}. We simulate the distributed setup with an Erdos-Renyi graph with $p=0.5$ and $M=20$ nodes. Both these datasets have $N=60,000$ samples distributed equally among the nodes, making $N_i=3000$. The data dimensions are $d=784$ for MNIST and $d=1024$ for CIFAR10. Figure~\ref{fig:a3} shows the comparison of the various PCA algorithms for MNIST dataset when $K=10$ dominant eigenvectors are estimated. The step size used for FAST-PCA and DSA in this case is $\alpha=0.1$. Similar results for CIFAR10 are shown in Figure~\ref{fig:b3} when $K=5$ and $\alpha = 0.8$ is used.

\begin{figure}
	\centering
	\begin{subfigure}{.23\textwidth}
		\centering
		\includegraphics[width=\linewidth]{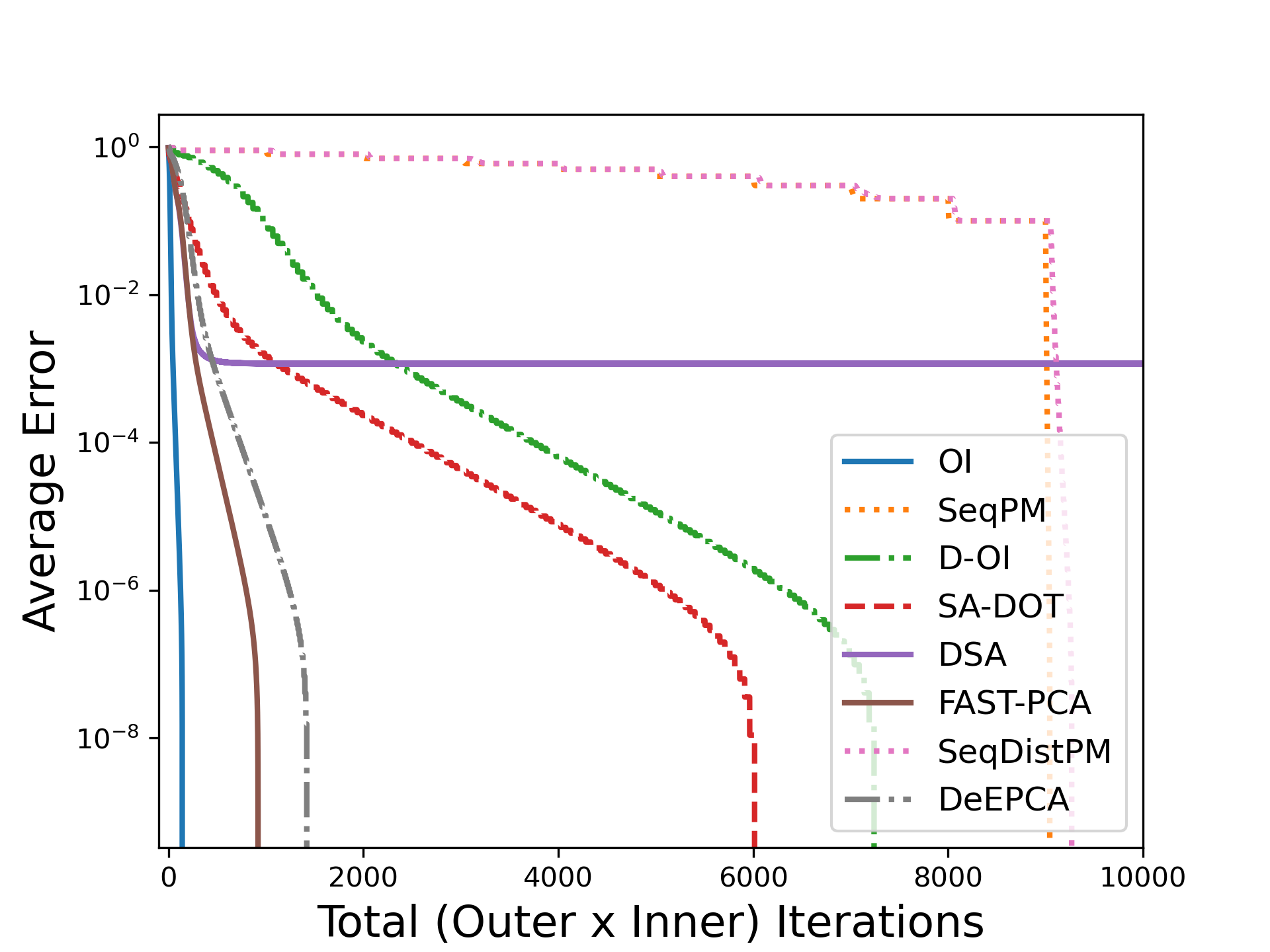}
		\caption{MNIST, $K=10$}
		\label{fig:a3}
	\end{subfigure}
	\hfil
	\begin{subfigure}{.23\textwidth}
		\centering
		\includegraphics[width=\linewidth]{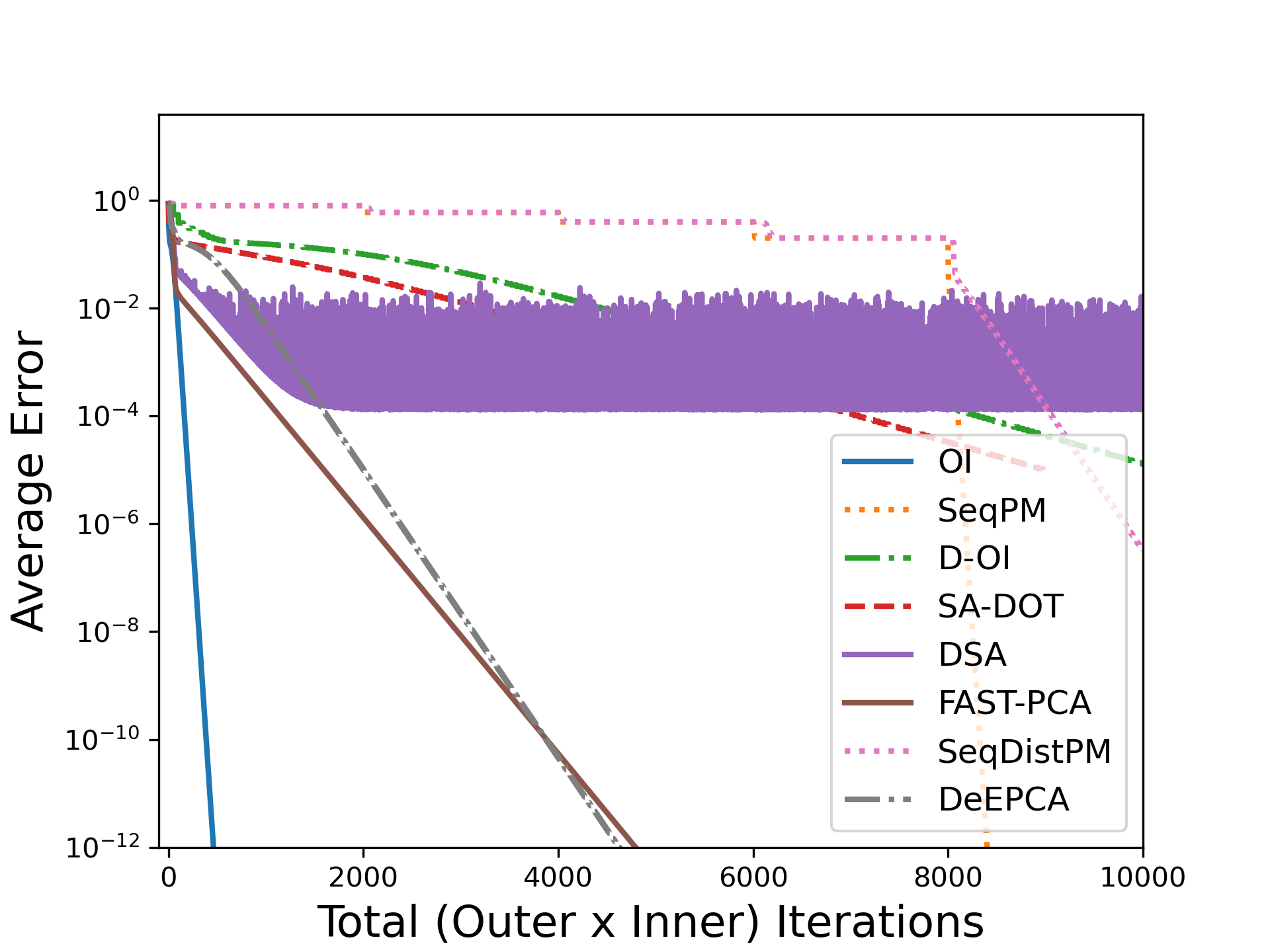}
		\caption{CIFAR10, $K = 5$}
		\label{fig:b3}
	\end{subfigure}
	\caption{Performance comparison of FAST-PCA with various algorithms for MNIST and CIFAR10}
	\label{fig:fast_comp3}
\end{figure} 
\section{Conclusion}\label{sec:conc}
In this paper, we proposed and analyzed a novel algorithm for distributed Principal Component Analysis (PCA) that truly serves the complete purpose of dimension reduction and uncorrelated feature learning in the scenario where data samples are distributed across a network. We provided detailed theoretical analysis to prove that our proposed algorithm converges linearly, exactly and globally, i.e., starting from any random unit vectors, to the eigenvectors of the global covariance matrix. We also provided experimental results that further validate our claims and demonstrate the communication efficiency and overall effectiveness of our solution. In the future, we aim to solve the problem of distributed PCA for estimation of multiple eigenvectors in the case of streaming data. Other possible directions are considering asynchronicity in the network and the case of directed and time-varying graphs.

\begin{appendices} 

\section{Statement and Proof of Lemma~\ref{lemma:coeff_decay_lower}}\label{app:lemma3}
\begin{lemma}\label{lemma:coeff_decay_lower}
	Suppose $z_{k,k}^{(0)} \neq 0$ and $\alpha <\frac{1}{\lambda_1}$. Then the following is true for $\gamma_k = \Big(\frac{1}{1 + \alpha\lambda_k }\Big)^2 < 1 $ and some constant $a_1 > 0$:
	\begin{equation}
	\sum_{l=1}^{k-1}({z}_{k,l}^{(t+1)})^2 \leq a_1\gamma_k^{t+1}.
	\end{equation}
\end{lemma}
\begin{proof}
For $l = 1,\ldots, k-1$ we know from \eqref{eq:eqn_z_upper} that ${z}_{k,l}^{(t+1)} = a_k^{(t)}\big(1 - \alpha(\tilde{\bx}_{g,k}^{(t)})^T\bC\tilde{\bx}_{g,k}^{(t)}\big){z}_{k,l}^{(t)}$. Since $(\tilde{\bx}_{g,k}^{(t)})^T\bC\tilde{\bx}_{g,k}^{(t)} \leq \lambda_1 < \frac{1}{\alpha}$, we have $1 + \alpha(\lambda_k - (\tilde{\bx}_{g,k}^{(t)})^T\bC\tilde{\bx}_{g,k}^{(t)}) >  \alpha\lambda_k \geq 0$.

Thus, we have for $l = 1, \cdots k-1$,
\begin{eqnarray*}
	\Bigg(\frac{{z}_{k,l}^{(t+1)}}{{z}_{k,k}^{(t+1)}}\Bigg)^2 &=& \Bigg(\frac{1 - \alpha (\tilde{\bx}_{g,k}^{(t)})^T\bC\tilde{\bx}_{g,k}^{(t)}}{1 +\alpha(\lambda_k - (\tilde{\bx}_{g,k}^{(t)})^T\bC\tilde{\bx}_{g,k}^{(t)})}\Bigg)^2\Bigg(\frac{{z}_{k,l}^{(t)}}{{z}_{k,k}^{(t)}}\Bigg)^2 \\
	&=&  \Bigg(1 - \frac{\alpha \lambda_k}{1 + \alpha (\lambda_k  - (\tilde{\bx}_{g,k}^{(t)})^T\bC\tilde{\bx}_{g,k}^{(t)})}\Bigg)^2\Bigg(\frac{{z}_{k,l}^{(t)}}{{z}_{k,k}^{(t)}}\Bigg)^2 \\
	&\leq&  \Big(1 - \frac{\alpha \lambda_k}{1 +\alpha\lambda_k }\Big)^2\Big(\frac{{z}_{k,l}^{(t)}}{{z}_{k,k}^{(t)}}\Big)^2 \\
	&=&  \Big(\frac{1  }{1 + \alpha\lambda_k }\Big)^2\Big(\frac{{z}_{k,l}^{(t)}}{{z}_{k,k}^{(t)}} \Big)^2 \\
	&=& \gamma_k\Big(\frac{{z}_{k,l}^{(t)}}{{z}_{k,k}^{(t)}} \Big)^2 , \quad \gamma_k  =  \Big(\frac{1}{1 + \alpha\lambda_k }\Big)^2 < 1.
\end{eqnarray*}
Therefore, for $l=1,\ldots,k-1$, $({z}_{k,l}^{(t+1)})^2 \leq \gamma_k^{t+1}\Big(\frac{{z}_{k,l}^{(0)}}{{z}_{k,k}^{(0)}} \Big)^2({z}_{k,k}^{(t+1)})^2$. 
Since $\|\tilde{\bx}_{k}^{(t+1)}\|^2 =1$ and $\|\tilde{\bx}_k^{(0)}\|^2 = 1$, hence $({z}_{k,k}^{t+1})^2 \leq 1$ and ${z}_{k,l}^{(0)} \leq 1$. Also, because of the assumption ${z}_{k,k}^{(0)} \neq 0$, let us assume $({z}_{k,k}^{(0)})^2 > \tilde{\eta}$. Thus, we can write
\begin{eqnarray}
	\sum_{l=1}^{k-1}({z}_{k,l}^{t+1})^2 &\leq& \gamma_k^{t+1}\sum_{l=1}^{k-1}\frac{1}{\tilde{\eta}} = a_1\gamma_k^{t+1}.
\end{eqnarray} 
\end{proof}

\section{Statement and Proof of Lemma~\ref{lemma:coeff_decay_upper}}\label{app:lemma4}
\begin{lemma}\label{lemma:coeff_decay_upper}
	Suppose $z_{k,k}^{(0)} \neq 0$ and $\alpha <\frac{1}{\lambda_1}$. Then the following is true for $\rho_k = \Big(\frac{1 + \alpha \lambda_{k+1}}{1 + \alpha\lambda_k }\Big)^2 < 1 $ and some constant $a_2 > 0$:
	\begin{equation}
	\sum_{l=k+1}^{d}({z}_{k,l}^{(t+1)})^2 \leq a_2\rho_k^{t+1}.
	\end{equation}
\end{lemma}
\begin{proof}
For $l = k,\ldots, d$ we know from \eqref{eq:eqn_z_upper} that ${z}_{k,l}^{(t+1)} = a_k^{(t)}\big(1 + \alpha(\lambda_l - (\tilde{\bx}_{g,k}^{(t)})^T\bC\tilde{\bx}_{g,k}^{(t)})\big){z}_{k,l}^{(t)}$. Since $ (\tilde{\bx}_{g,k}^{(t)})^T\bC\tilde{\bx}_{g,k}^{(t)} \leq \lambda_1 < \frac{1}{\alpha}$, we have $1 + \alpha(\lambda_l - (\tilde{\bx}_{g,k}^{(t)})^T\bC\tilde{\bx}_{g,k}^{(t)}) >  \alpha\lambda_l \geq 0, \forall l = k,\ldots, d$.

Thus, we have for $l = k+1, \cdots d$,
\begin{eqnarray*}
	\Bigg(\frac{{z}_{k,l}^{(t+1)}}{{z}_{k,k}^{(t+1)}}\Bigg)^2 &=& \Bigg(\frac{1 + \alpha (\lambda_l - (\tilde{\bx}_{g,k}^{(t)})^T\bC\tilde{\bx}_{g,k}^{(t)})}{1 +\alpha(\lambda_k - (\tilde{\bx}_{g,k}^{(t)})^T\bC\tilde{\bx}_{g,k}^{(t)})}\Bigg)^2\Bigg(\frac{{z}_{k,l}^{(t)}}{{z}_{k,k}^{(t)}}\Bigg)^2 \\
	&=&  \Bigg(1 - \frac{\alpha (\lambda_k-\lambda_l)}{1 + \alpha (\lambda_k  - (\tilde{\bx}_{g,k}^{(t)})^T\bC\tilde{\bx}_{g,k}^{(t)})}\Bigg)^2\Bigg(\frac{{z}_{k,l}^{(t)}}{{z}_{k,k}^{(t)}}\Bigg)^2 \\
	&\leq&  \Big(1 - \frac{\alpha (\lambda_k-\lambda_l)}{1 +\alpha\lambda_k }\Big)^2\Big(\frac{{z}_{k,l}^{(t)}}{{z}_{k,k}^{(t)}}\Big)^2 \\
	&=&  \Big(\frac{1 + \alpha \lambda_l }{1 + \alpha\lambda_k }\Big)^2\Big(\frac{{z}_{k,l}^{(t)}}{{z}_{k,k}^{(t)}} \Big)^2 \leq  \Big(\frac{1 + \alpha \lambda_{k+1}}{1 + \alpha\lambda_k }\Big)^2\Big(\frac{{z}_{k,l}^{(t)}}{{z}_{k,k}^{(t)}} \Big)^2 \\
	&=& \rho_k\Big(\frac{{z}_{k,l}^{(t)}}{{z}_{k,k}^{(t)}} \Big)^2 , \quad \rho_k  =  \Big(\frac{1 + \alpha \lambda_{k+1}}{1 + \alpha\lambda_k }\Big)^2 < 1.
\end{eqnarray*}
Therefore, for $l=k+1,\ldots,d$, $({z}_{k,l}^{(t+1)})^2 \leq \rho_k^{t+1}\Big(\frac{{z}_{k,l}^{(0)}}{{z}_{k,k}^{(0)}} \Big)^2({z}_{k,k}^{(t+1)})^2$. 
Since $\|\tilde{\bx}_{k}^{(t+1)}\|^2 =1$ and $\|\tilde{\bx}_k^{(0)}\|^2 = 1$, hence $({z}_{k,k}^{t+1})^2 \leq 1$ and ${z}_{k,l}^{(0)} \leq 1$. Also, since ${z}_{k,k}^{(0)} \neq 0$, let us assume $({z}_{k,k}^{(0)})^2 > \tilde{\eta}$. Thus, we can write
\begin{eqnarray}
	\sum_{l=k+1}^{d}({z}_{k,l}^{t+1})^2 &\leq& \rho_k^{t+1}\sum_{l=k+1}^{d}\frac{1}{\tilde{\eta}} = a_2\rho_k^{t+1}.
\end{eqnarray} 
\end{proof}

\section{Proof of Lemma~\ref{lemma:lipshitz_continuous}}\label{app:lemma_lipschitz}
For a continuous and differentiable function $f:\R^d \rightarrow \R^d$, we know $\|f(x)-f(y)\| \leq \|\nabla f(x)\|\|x-y\|$. Thus, the Lipschitz constant can be given by the upper bound of $\|\nabla f(x)\|$. For the following function,
\begin{align*}
    \bh_i(\bv) = \bC_i\bv - \frac{(\bv)^T\bC_i\bv}{\|\bv\|^2}\bv, 
\end{align*}
taking derivative on both sides gives
\begin{align*}
    \frac{\partial}{\partial\bv}\bh_i(\bv) &= \bC_i - \frac{\partial}{\partial\bv}\Big(\frac{(\bv)^T\bC_i\bv}{\|\bv\|^2}\Big)\bv^T - \frac{(\bv)^T\bC_i\bv}{\|\bv\|^2}\bI  \\
    &= \bC_i - \frac{2\|\bv\|^2\bC_i\bv - 2(\bv)^T\bC_i\bv\bv}{\|\bv\|^4}\bv^T - \frac{(\bv)^T\bC_i\bv}{\|\bv\|^2}\bI  \\
    &= \bC_i - \frac{2\bC_i\bv\bv^T}{\|\bv\|^2} + \frac{2(\bv)^T\bC_i\bv\bv\bv^T}{\|\bv\|^4} - \frac{(\bv)^T\bC_i\bv}{\|\bx_{i,1}\|^2}\bI. \\
    \|\frac{\partial}{\partial\bv}\bh_i(\bv)\| &\leq \|\bC_i\| + \|\frac{2\bC_i\bv\bv^T}{\|\bv\|^2}\| + \|\frac{2(\bv)^T\bC_i\bv\bv\bv^T}{\|\bv\|^4}\| \\
    &\qquad \qquad \qquad \qquad \qquad + \|\frac{(\bv)^T\bC_i\bv}{\|\bv\|^2}\bI\| \\
    &\leq \lambda_{i,1} + 2\frac{\lambda_{i,1}\|\bv\|^2}{\|\bv\|^2} + 2|(\bv)^T\bC_i\bv|\frac{\|\bv\|^2}{\|\bv\|^4} \\
    & + \frac{|(\bv)^T\bC_i\bv|}{\|\bv\|^2} \quad \text{where,} \quad \lambda_{i,1} = \|\bC_i\|\\
    &\leq \lambda_{i,1} + 2\lambda_{i,1} + 2\lambda_{i,1} + \lambda_{i,1} = 6\lambda_{i,1}.
\end{align*}
Thus,
\begin{align}
    \|\bh_i(\bv_1)-\bh_i(\bv_2)\| \leq 6\lambda_{i,1}\|\bv_1 - \bv_2\| \leq 6\lambda_{1}\|\bv_1 - \bv_2\|,
\end{align}
where the last inequality uses the fact that $\bC_i \preceq \bC$, hence $\lambda_{i,1} \leq \lambda_1$.
\qed

\section{Proof of Lemma~\ref{lemma:inequalities}}\label{app:lemma_inequalities}
This lemma uses Lemma~\ref{lemma:lipshitz_continuous} to prove three inequalities that aid in the proof of Theorem~\ref{theorem:theorem1} and Theorem~\ref{theorem:theoremk}.
\begin{proof}
     \textit{1.} First, we prove the Lipschitz continuity of the stacked function $\bh(\bx_1^{(t)})$.
	\begin{align*}
		\|\bh(\bx_1^{(t)}) - \bh(\bx_1^{(t-1)})\|_2 ^2 &= \sum_{i=1}^{M}\|\bh_i(\bx_{i,1}^{(t)}) - \bh_i(\bx_{i,1}^{(t-1)})\|^2\\
		&\leq L_1^2\sum_{i=1}^{M}\|\bx_{i,1}^{(t)}- {\bx}_{i,1}^{(t-1)}\|^2\\
		&= L_1^2\|\bx_{1}^{(t)}- {\bx}_{1}^{(t-1)}\|^2\\
		\|\bh(\bx_1^{(t)}) - \bh(\bx_1^{(t-1)})\|_2  &\leq L_1\|\bx_{1}^{(t)}- {\bx}_{1}^{(t-1)}\|.
	\end{align*}
\textit{2.} Here, we prove the Lipschitz continuity of the function $\bg(\bx_{1}^{(t)}) $.
\begin{align*}
	\|\bg(\bx_{1}^{(t)}) - \bg(\bx_{1}^{(t-1)})\|_2^2 &= \frac{1}{M^2}\|\sum_{i=1}^{M}(\bh_i(\bx_{i,1}^{(t)})- \bh_i(\bx_{i,1}^{(t-1)}))\|^2\\
	&\leq \frac{1}{M^2}M\sum_{i=1}^{M}\|\bh_i(\bx_{i,1}^{(t)})- \bh_i(\bx_{i,1}^{(t-1)})\|^2 \\
	&\leq \frac{L_1^2}{M}\sum_{i=1}^{M}\|\bx_{i,1}^{(t)} - \bx_{i,1}^{(t-1)}\|^2\\
	&= \frac{L_1^2}{M}\|\bx_{1}^{(t)} - \bx_{1}^{(t-1)}\|^2\\
	\|\bg(\bx_{1}^{(t)}) - \bg(\bx_{1}^{(t-1)})\|_2 &\leq \frac{L_1}{\sqrt{M}}\|\bx_{1}^{(t)} - \bx_{1}^{(t-1)}\|_2.
\end{align*}
\textit{3.} Using the Lipschitz continuous property of $\bg(\bx_{1}^{(t)})$, we get
\begin{align*}
\|\bg(\bx_{1}^{(t)}) - {\bg}(\bar{\bx}_{s,1}^{(t)})\|_2 &\leq \frac{L_1}{\sqrt{M}}\|\bx_{1}^{(t)} -  \bar{\bx}_{s,1}^{(t)}\|_2.
\end{align*}
\end{proof}

\section{Statement and Proof of Lemma~\ref{lemma:boundP}}\label{app:boundP}
\begin{lemma}\label{lemma:boundP}
    For a matrix $\bP_k(\alpha)$ such that $$\bP(\alpha) = \begin{bmatrix}
		(\frac{1+\beta}{2} + \alpha L_k) & L_k(2+\alpha L_k) & \alpha L_k^2\\
		\alpha & \frac{1+\beta}{2} & 0\\
		0 & \alpha L_k & \delta_k
	\end{bmatrix}$$
	where $L_k = (k+5)\lambda_1$ and $\delta_k = \frac{1+\alpha\lambda_{k+1}}{1+\alpha\lambda_k}$, the spectral radius $\rho(\bP_k(\alpha))$ is strictly less than 1 if $\alpha < \frac{\lambda_1-\lambda_2}{42}(\frac{1-\beta}{9\lambda_1})^2$.
\end{lemma}
\begin{proof}
Since $\bP_k(\alpha)$ is a non-negative matrix, by Perron-Frobenius theorem its characteristic polynomial has a simple positive real root $r$ such that $\rho(\bP_k(\alpha)) = r$. We know $\delta_k = \frac{1 + \alpha\lambda_{k+1}}{1+ \alpha\lambda_k} < 1$ and $L_k = (k+5)\lambda_1$. 
Now, the characteristic polynomial $p(\gamma)$ of $\bP(\alpha)$ is given as
\begin{align*}
	p(\gamma) &= |\gamma\bI - \bP(\alpha)| \\
	&= \begin{vmatrix}
		\gamma - (\frac{1+\beta}{2} + \alpha L_k) & -L_k(\alpha L_k +2) & - \alpha L_k^2\\
		-\alpha& \gamma - \frac{1+\beta}{2} & 0\\
		0 & - \alpha L_k & \gamma - \delta_1
	\end{vmatrix}\\
	&= (\gamma - \frac{1+\beta}{2} - \alpha L_k)(\gamma - \frac{1+\beta}{2})(\gamma - \delta_k) + \\
	& \qquad\qquad\qquad\qquad \alpha(-L_k(\alpha L_k +2)(\gamma - \delta_k) - \alpha^2 L_k^3)\\
	&= \big((\gamma - \frac{1+\beta}{2} - \alpha L_k)(\gamma - \frac{1+\beta}{2}) - \\
	& \qquad\qquad\qquad\qquad \alpha L_k(\alpha L_k +2) \big)(\gamma - \delta_k) - \alpha^3 L_k^3 \\
	&= p_0(\gamma)(\gamma - \delta_k) - \alpha^3 L_k^3,
\end{align*}
where
\begin{align*}
	p_0(\gamma) &= (\gamma - \frac{1+\beta}{2} - \alpha L_k)(\gamma - \frac{1+\beta}{2}) - \alpha L_k(\alpha L_k +2)\\
	&= \gamma^2 - (1+ \beta + \alpha L_k)\gamma + \frac{1+\beta}{2}(\frac{1+\beta}{2} + \alpha L_k) - \alpha L_k(\alpha L_k +2) \\
	&= (\gamma - \gamma_1)(\gamma - \gamma_2),
\end{align*}
with $\gamma_1, \gamma_2$ being the roots of $p_0(\gamma)$, given as
\begin{equation}
	\gamma_1, \gamma_2 = \frac{1+\beta + \alpha L_k \pm \sqrt{5\alpha^2L_k^2 + 8\alpha L_k}}{2}.
\end{equation}
If $0 < \alpha < \frac{1}{L_k}$, $\alpha L_k < 1$ and it implies $\alpha^2 L_k^2 < \alpha L_k < \sqrt{\alpha L_k}$. Thus,
\begin{eqnarray*}
	\gamma_1, \gamma_2 &=& \frac{1+\beta + \alpha L_k \pm \sqrt{5\alpha^2L_k^2 + 8\alpha L_k}}{2} \\
	&<& \frac{1+\beta + \alpha L_k + \sqrt{5\alpha^2L_k^2 + 8\alpha L_k}}{2} \\
	&<& \frac{1+\beta + \sqrt{\alpha L_k} + \sqrt{5\alpha L_k + 8\alpha L_k}}{2} \\
	&<& \frac{1+\beta + \sqrt{\alpha L_k} + \sqrt{16\alpha L_k}}{2} \\
	&=& \frac{1+\beta + 5\sqrt{\alpha L_k} }{2}  = \gamma_0.
\end{eqnarray*}
For $\gamma \geq  \gamma_0 $, $p_0(\gamma) \geq (\gamma - \gamma_0 )^2$. Now, 
let $\gamma^* = \max \{ 1- \frac{1}{2}\frac{\alpha(\lambda_k - \lambda_{k+1})}{1+\alpha\lambda_k}, \frac{1+\beta}{2} + 4.5\sqrt{\alpha L_1}\sqrt{\frac{(1+\alpha\lambda_k)L_1}{\lambda_k - \lambda_{k+1}}} \} > \gamma_0$. Then
\begin{align*}
	p(\gamma^*) &\geq \frac{1}{2} \frac{\alpha(\lambda_k - \lambda_{k+1})}{1+\alpha\lambda_k} (4.5\sqrt{\alpha L_k}\sqrt{\frac{(1+\alpha\lambda_k)L_k}{\lambda_k-\lambda_{k+1}}} \\
	& \qquad\qquad\qquad - 2.5\sqrt{\alpha L_k})^2 - \alpha^3 L_k^3\\
	&\geq \frac{1}{2} \frac{\alpha(\lambda_k - \lambda_{k+1})}{1+\alpha\lambda_{k+1}} \big(4.5\sqrt{\alpha L_k}\sqrt{\frac{(1+\alpha\lambda_k)L_k}{\lambda_k-\lambda_{k+1}}} \\
	& \qquad\qquad\qquad -  2.5\sqrt{\alpha L_k}\sqrt{\frac{(1+\alpha\lambda_k)L_k}{\lambda_k-\lambda_{k+1}}}\big)^2 - \alpha^3 L_k^3\\
	&\geq \frac{1}{2} \frac{\alpha(\lambda_k - \lambda_{k+1})}{1+\alpha\lambda_k} \big(2\sqrt{\alpha L_k}\sqrt{\frac{(1+\alpha\lambda_k)L_k}{\lambda_k-\lambda_{k+1}}} \big)^2 -\alpha^3 L_k^3\\
	&= 2\alpha^2 L_k^2 - \alpha^3 L_k^3 \geq 0.
\end{align*}
Since $p(\gamma) = p_0(\gamma)(\gamma - \delta_k) - \alpha^3 L_k^3 \geq (\gamma - \gamma_0 )^2(\gamma - \delta_k) - \alpha^3 L_k^3$, evidently it is a strictly increasing function on $ [{\max \{\delta_k, \gamma_0\} , +\infty} )$ and since this interval includes $\gamma^*$, $p(\gamma)$ has no real roots on $(\gamma^*,+\infty)$. Thus, the real root of the characteristic polynomial is $\leq \gamma^*$. Hence $\rho(\bP_k(\alpha)) \leq \gamma^*$. If we choose $\alpha$ such that $\gamma^* < 1$, then $\rho(\bP_k(\alpha)) < 1$ and the convergence would be linear. For $\gamma^* < 1$, we need
\begin{align*}
	\frac{1+\beta}{2} + 4.5\sqrt{\alpha L_k}\sqrt{\frac{(1+\alpha\lambda_k)L_k}{\lambda_k-\lambda_{k+1}}} &< 1 \\
	\sqrt{\alpha L_k}\sqrt{\frac{(1+\alpha\lambda_k)L_k}{\lambda_k-\lambda_{k+1}}} &< \frac{1 - \beta}{9} \\
	(\alpha L_k)\frac{(1+\alpha\lambda_k)L_k}{\lambda_k-\lambda_{k+1}} &< (\frac{1 - \beta}{9})^2\\
	\alpha (1+\alpha\lambda_k) &< \frac{\lambda_k - \lambda_{k+1}}{L_k^2}(\frac{1 - \beta}{9})^2 \\
	&= \frac{\lambda_k - \lambda_{k+1}}{\lambda_k^2(k+5)^2}(\frac{1 - \beta}{9})^2.
\end{align*} 
If $\alpha < \frac{1}{L_k} = \frac{1}{(k+5)\lambda_1} \leq \frac{1}{(k+5)\lambda_k}$, then $\alpha\lambda_k < \frac{1}{k+5}$. Thus, $1+\alpha\lambda_k < \frac{k+6}{k+5}$.
If $\alpha < \frac{\lambda_k-\lambda_{k+1}}{(k+5)(k+6)}(\frac{1-\beta}{9\lambda_1})^2 < \frac{1}{L_k} = \frac{1}{(k+5)\lambda_1}$, that is $\frac{k+6}{k+5}\alpha < \frac{\lambda_k-\lambda_{k+1}}{\lambda_1^2(k+5)^2}(\frac{1-\beta}{9})^2$,
then $\alpha (1+\alpha\lambda_k) < \frac{k+6}{k+5}\alpha < \frac{\lambda_k-\lambda_{k+1}}{\lambda_1^2(k+5)^2}(\frac{1-\beta}{9})^2$.
\end{proof}

\section{Proof of Lemma~\ref{lemma:lipshitz_continuousk}}\label{app:lemma_lipschitzk}
For a continuous function $f:\R^d \rightarrow \R^d$, we know $\|f(x)-f(y)\| \leq \|\nabla f(x)\|\|x-y\|$. Thus, the Lipschitz constant can be given by the upper bound of $\|\nabla f(x)\|$. Taking derivative on both sides of the function
\begin{align*}
    \bh_{i,t}(\bv) = \bC_i\bv - \frac{(\bv)^T\bC_i\bv}{\|\bv\|^2}\bv - - \sum_{p=1}^{k-1}\frac{\bx_{i,p}^{(t)}(\bx_{i,p}^{(t)})^T}{\|\bx_{i,p}^{(t)}\|^2}\bC_i\bv
\end{align*}
we get
\begin{align*}
    \frac{\partial}{\partial\bv}\bh_{i,t}(\bv) &= \bC_i - \frac{\partial}{\partial\bv}\Big(\frac{(\bv)^T\bC_i\bv}{\|\bv\|^2}\Big)\bv^T - \frac{(\bv)^T\bC_i\bv}{\|\bv\|^2}\bI \\
    &\qquad\qquad\qquad\qquad - \Big(\sum_{p=1}^{k-1}\frac{\bx_{i,p}^{(t)}(\bx_{i,p}^{(t)})^T}{\|\bx_{i,p}^{(t)}\|^2}\bC_i\Big)^T \\
    &= \bC_i - \frac{2\|\bv\|^2\bC_i\bv - 2(\bv)^T\bC_i\bv\bv}{\|\bv\|^4}\bv^T - \frac{(\bv)^T\bC_i\bv}{\|\bv\|^2}\bI \\
    & \qquad\qquad\qquad\qquad - \Big(\sum_{p=1}^{k-1}\frac{\bx_{i,p}^{(t)}(\bx_{i,p}^{(t)})^T}{\|\bx_{i,p}^{(t)}\|^2}\bC_i\Big)^T \\
    &= \bC_i - \frac{2\bC_i\bv\bv^T}{\|\bv\|^2} + \frac{2(\bv)^T\bC_i\bv\bv\bv^T}{\|\bv\|^4} - \frac{(\bv)^T\bC_i\bv}{\|\bv\|^2}\bI \\
    &\qquad\qquad\qquad\qquad - \sum_{p=1}^{k-1}\bC_i\frac{\bx_{i,p}^{(t)}(\bx_{i,p}^{(t)})^T}{\|\bx_{i,p}^{(t)}\|^2} \\
    \|\frac{\partial}{\partial\bv}\bh_{i,t}(\bv)\| &\leq \|\bC_i\| + \|\frac{2\bC_i\bv\bv^T}{\|\bv\|^2}\| + \|\frac{2(\bv)^T\bC_i\bv\bv\bv^T}{\|\bv\|^4}\| +\\
    & \qquad\qquad \|\frac{(\bv)^T\bC_i\bv}{\|\bv\|^2}\bI\| + \sum_{p=1}^{k-1}\|\bC_i\frac{\bx_{i,p}^{(t)}(\bx_{i,p}^{(t)})^T}{\|\bx_{i,p}^{(t)}\|^2}\|\\
    &\leq \lambda_{i,1} + 2\frac{\lambda_{i,1}\|\bv\|^2}{\|\bv\|^2} + 2|(\bv)^T\bC_i\bv|\frac{\|\bv\|^2}{\|\bv\|^4} + \\
    & \qquad\qquad \frac{|(\bv)^T\bC_i\bv|}{\|\bv\|^2} + \sum_{p=1}^{k-1}\|\bC_i\|\|\frac{\bx_{i,p}^{(t)}(\bx_{i,p}^{(t)})^T}{\|\bx_{i,p}^{(t)}\|^2}\| \\
    &\leq \lambda_{i,1} + 2\lambda_{i,1} + 2\lambda_{i,1} + \lambda_{i,1} + \lambda_{i,1}(k-1)\\
    &= (k+5)\lambda_{i,1} 
\end{align*}
Thus,
\begin{align*}
    \|\bh_{i,t}(\bv_1)-\bh_{i,t}(\bv_2)\| &\leq \lambda_{i,1}(k+5)\|\bv_1 - \bv_2\| \\
    &\leq \lambda_{1}(k+5)\|\bv_1 - \bv_2\|
\end{align*}
\qed

\end{appendices}

\balance


\begin{thebibliography}{10}
	\providecommand{\url}[1]{#1}
	\csname url@samestyle\endcsname
	\providecommand{\newblock}{\relax}
	\providecommand{\bibinfo}[2]{#2}
	\providecommand{\BIBentrySTDinterwordspacing}{\spaceskip=0pt\relax}
	\providecommand{\BIBentryALTinterwordstretchfactor}{4}
	\providecommand{\BIBentryALTinterwordspacing}{\spaceskip=\fontdimen2\font plus
		\BIBentryALTinterwordstretchfactor\fontdimen3\font minus
		\fontdimen4\font\relax}
	\providecommand{\BIBforeignlanguage}[2]{{%
			\expandafter\ifx\csname l@#1\endcsname\relax
			\typeout{** WARNING: IEEEtran.bst: No hyphenation pattern has been}%
			\typeout{** loaded for the language `#1'. Using the pattern for}%
			\typeout{** the default language instead.}%
			\else
			\language=\csname l@#1\endcsname
			\fi
			#2}}
	\providecommand{\BIBdecl}{\relax}
	\BIBdecl
	
	\bibitem{gang.raja.bajwa.2019}
	A.~{Gang}, H.~{Raja}, and W.~U. {Bajwa}, ``Fast and communication-efficient
	distributed {PCA},'' in \emph{Proc. 2019 IEEE International Conf. Acoustics,
		Speech and Signal Process. (ICASSP)}, 2019, pp. 7450--7454.
	
	\bibitem{Hotelling.1933}
	H.~Hotelling, ``Analysis of a complex of statistical variables into principal
	components.'' \emph{J. Educational Psychology}, vol.~24, no.~6, pp. 417--441,
	1933.
	
	\bibitem{BajwaCevherEtAl.ISPM20}
	W.~U. Bajwa, V.~Cevher, D.~Papailiopoulos, and A.~Scaglione, ``Machine learning
	from distributed, streaming data,'' \emph{IEEE Signal Process. Mag.},
	vol.~37, no.~3, pp. 11--13, May 2020.
	
	\bibitem{Yang.Gang.Bajwa.2020}
	Z.~{Yang}, A.~{Gang}, and W.~U. {Bajwa}, ``Adversary-resilient distributed and
	decentralized statistical inference and machine learning: An overview of
	recent advances under the {B}yzantine threat model,'' \emph{IEEE Signal
		Process. Mag.}, vol.~37, no.~3, pp. 146--159, 2020.
	
	\bibitem{baldi.hornik.1989}
	P.~Baldi and K.~Hornik, ``Neural networks and principal component analysis:
	Learning from examples without local minima,'' \emph{Neural Netw.}, vol.~2,
	no.~1, p. 53–58, Jan. 1989.
	
	\bibitem{oja}
	E.~Oja and J.~Karhunen, ``On stochastic approximation of the eigenvectors and
	eigenvalues of the expectation of a random matrix,'' \emph{J. Math. Anal.
		Applicat.}, vol. 106, no.~1, pp. 69 -- 84, 1985.
	
	\bibitem{sanger}
	T.~D. Sanger, ``Optimal unsupervised learning in a single-layer linear
	feedforward neural network,'' \emph{Neural Netw.}, vol.~2, no.~6, pp. 459 --
	473, 1989.
	
	\bibitem{Hebb.1949}
	D.~O. Hebb, \emph{\BIBforeignlanguage{English}{The Organization of Behavior : A
			Neuropsychological Theory}}.\hskip 1em plus 0.5em minus 0.4em\relax Wiley New
	York, 1949.
	
	\bibitem{krasulina}
	T.~P. {Krasulina}, ``\BIBforeignlanguage{English}{{Method of stochastic
			approximation in the determination of the largest eigenvalue of the
			mathematical expectation of random matrices}},''
	\emph{\BIBforeignlanguage{English}{{Autom. Remote Control}}}, vol. 1970, pp.
	215--221, 1970.
	
	\bibitem{balsubramani2013}
	A.~Balsubramani, S.~Dasgupta, and Y.~Freund, ``The fast convergence of
	incremental {PCA},'' in \emph{Advances in Neural Information Processing
		Systems}, C.~J.~C. Burges, L.~Bottou, M.~Welling, Z.~Ghahramani, and K.~Q.
	Weinberger, Eds., vol.~26.\hskip 1em plus 0.5em minus 0.4em\relax Curran
	Associates, Inc., 2013.
	
	\bibitem{matrix_krasulina2019}
	C.~Tang, ``Exponentially convergent stochastic k-{PCA} without variance
	reduction,'' in \emph{NeurIPS}, 2019.
	
	\bibitem{Pearson.1901}
	K.~Pearson, ``On lines and planes of closest fit to systems of points in
	space,'' \emph{Philosophical Mag.}, vol.~2, pp. 559--572, 1901.
	
	\bibitem{Golub}
	G.~H. Golub and C.~F. Van~Loan, \emph{Matrix Computations (3rd Ed.)}.\hskip 1em
	plus 0.5em minus 0.4em\relax Baltimore, MD, USA: Johns Hopkins University
	Press, 1996.
	
	\bibitem{Lanczos.1950}
	C.~Lanczos, ``An iteration method for the solution of the eigenvalue problem of
	linear differential and integral operators,'' \emph{J. Research Nat. Bureau
		Standards}, 1950.
	
	\bibitem{yi.tan2005}
	{Z. Yi}, {M. Ye}, {J. C. Lv}, and {K. K. Tan}, ``Convergence analysis of a
	deterministic discrete time system of {O}ja's {PCA} learning algorithm,''
	\emph{IEEE Trans. Neural Netw.}, vol.~16, no.~6, pp. 1318--1328, Nov 2005.
	
	\bibitem{lv.yi.tan.2007}
	J.~C. {Lv}, Z.~{Yi}, and K.~K. {Tan}, ``Global convergence of {GHA} learning
	algorithm with nonzero-approaching adaptive learning rates,'' \emph{IEEE
		Trans. Neural Netw.}, vol.~18, no.~6, pp. 1557--1571, 2007.
	
	\bibitem{wu2018review}
	S.~X. Wu, H.-T. Wai, L.~Li, and A.~Scaglione, ``A review of distributed
	algorithms for principal component analysis,'' \emph{Proc. IEEE}, vol. 106,
	no.~8, pp. 1321--1340, 2018.
	
	\bibitem{mcsherry}
	D.~Kempe and F.~McSherry, ``A decentralized algorithm for spectral analysis,''
	\emph{J. Comput. and Syst. Sci.}, vol.~74, no.~1, pp. 70 -- 83, 2008.
	
	\bibitem{scaglione.krim.2008}
	A.~{Scaglione}, R.~{Pagliari}, and H.~{Krim}, ``The decentralized estimation of
	the sample covariance,'' in \emph{Proc. 2008 42nd Asilomar Conf. on Signals,
		Syst. and Comput.}, 2008, pp. 1722--1726.
	
	\bibitem{d-oja}
	L.~Li, A.~Scaglione, and J.~H. Manton, ``Distributed principal subspace
	estimation in wireless sensor networks,'' \emph{IEEE J. Sel. Topics Signal
		Process.}, vol.~5, no.~4, pp. 725--738, Aug 2011.
	
	\bibitem{cksvd.allerton.2013}
	H.~{Raja} and W.~U. {Bajwa}, ``Cloud {K-SVD}: Computing data-adaptive
	representations in the cloud,'' in \emph{Proc. 2013 51st Annual Allerton
		Conf. Commun., Control and Computing (Allerton)}, 2013, pp. 1474--1481.
	
	\bibitem{cksvd}
	H.~Raja and W.~U. Bajwa, ``Cloud-{K-SVD}: A collaborative dictionary learning
	algorithm for big, distributed data,'' \emph{IEEE Trans. Signal Process.},
	vol.~64, no.~1, pp. 173--188, Jan 2016.
	
	\bibitem{depm}
	H.~Wai, A.~Scaglione, J.~Lafond, and E.~Moulines, ``Fast and privacy preserving
	distributed low-rank regression,'' in \emph{Proc. {IEEE} Int. Conf.
		Acoustics, Speech and Signal Process., (ICASSP)}, 2017, pp. 4451--4455.
	
	\bibitem{consensus}
	L.~Xiao and S.~Boyd, ``Fast linear iterations for distributed averaging,''
	\emph{Syst. \& Control Letters}, vol.~53, no.~1, pp. 65--78, 2004.
	
	\bibitem{xiang.gang.bajwa.2021}
	A.~Gang, B.~Xiang, and W.~U. Bajwa, ``Distributed principal subspace analysis
	for partitioned big data: Algorithms, analysis, and implementation,''
	\emph{IEEE Transactions on Signal and Information Processing over Networks},
	vol.~7, pp. 699--715, 2021.
	
	\bibitem{proxpda}
	M.~Hong, D.~Hajinezhad, and M.-M. Zhao, ``Prox-{PDA}: The proximal primal-dual
	algorithm for fast distributed nonconvex optimization and learning over
	networks,'' in \emph{Proc. 34th Int. Conf. Mach. Learning}, vol.~70.\hskip
	1em plus 0.5em minus 0.4em\relax PMLR, 06--11 Aug 2017, pp. 1529--1538.
	
	\bibitem{bianchi.jakubowicz.2013}
	P.~{Bianchi} and J.~{Jakubowicz}, ``Convergence of a multi-agent projected
	stochastic gradient algorithm for non-convex optimization,'' \emph{IEEE
		Trans. Autom. Control}, vol.~58, no.~2, pp. 391--405, 2013.
	
	\bibitem{wai.scaglione.lafond.2016}
	H.~{Wai}, A.~{Scaglione}, J.~{Lafond}, and E.~{Moulines}, ``A projection-free
	decentralized algorithm for non-convex optimization,'' in \emph{Proc. 2016
		IEEE Global Conf. Signal and Inform. Process. (GlobalSIP)}, 2016, pp.
	475--479.
	
	\bibitem{chen.mingyi.2021}
	S.~Chen, A.~Garcia, M.~Hong, and S.~Shahrampour, ``Decentralized riemannian
	gradient descent on the {S}tiefel manifold,'' \emph{arXiv preprint
		arXiv:2102.07091}, 2021.
	
	\bibitem{picard}
	F.~L. Andrade, M.~A. Figueiredo, and J.~Xavier, ``Distributed {P}icard
	iteration,'' \emph{arXiv preprint arXiv:2104.00131}, 2021.
	
	\bibitem{picardPCA}
	------, ``Distributed {P}icard iteration: Application to distributed {EM} and
	distributed {PCA},'' \emph{arXiv preprint arXiv:2106.10665}, 2021.
	
	\bibitem{gang.bajwa.2021}
	\BIBentryALTinterwordspacing
	A.~Gang and W.~U. Bajwa, ``A linearly convergent algorithm for distributed
	principal component analysis,'' \emph{Signal Processing}, vol. 193, p.
	108408, 2022. [Online]. Available:
	\url{https://www.sciencedirect.com/science/article/pii/S016516842100445X}
	\BIBentrySTDinterwordspacing
	
	\bibitem{dgd}
	A.~Nedic and A.~Ozdaglar, ``Distributed subgradient methods for multi-agent
	optimization,'' \emph{IEEE Trans. Autom. Control}, vol.~54, no.~1, pp.
	48--61, Jan 2009.
	
	\bibitem{extra}
	W.~Shi, Q.~Ling, G.~Wu, and W.~Yin, ``{EXTRA:} an exact first-order algorithm
	for decentralized consensus optimization,'' \emph{SIAM J. Optim.}, vol.~25,
	no.~2, pp. 944--966, 2015.
	
	\bibitem{qu.li.2018}
	G.~Qu and N.~Li, ``Harnessing smoothness to accelerate distributed
	optimization,'' \emph{IEEE Trans. Control Netw. Syst.}, vol.~5, no.~3, pp.
	1245--1260, 2018.
	
	\bibitem{next}
	P.~D. {Lorenzo} and G.~{Scutari}, ``{NEXT}: {I}n-network nonconvex
	optimization,'' \emph{IEEE Trans. Signal Inform. Process. Netw.}, vol.~2,
	no.~2, pp. 120--136, 2016.
	
	\bibitem{deepca}
	\BIBentryALTinterwordspacing
	H.~Ye and T.~Zhang, ``De{EPCA}: Decentralized exact {PCA} with linear
	convergence rate,'' \emph{Journal of Machine Learning Research}, vol.~22, no.
	238, pp. 1--27, 2021. [Online]. Available:
	\url{http://jmlr.org/papers/v22/21-0298.html}
	\BIBentrySTDinterwordspacing
	
	\bibitem{arora2013stochastic}
	R.~Arora, A.~Cotter, and N.~Srebro, ``Stochastic optimization of {PCA} with
	capped {MSG},'' in \emph{Advances Neural Inform. Process. Systs.}, 2013, pp.
	1815--1823.
	
	\bibitem{Boydfastestmixing.2003}
	S.~Boyd, P.~Diaconis, and L.~Xiao, ``Fastest mixing {M}arkov chain on a
	graph,'' \emph{SIAM REVIEW}, vol.~46, pp. 667--689, 2003.
	
	\bibitem{mnist.2010}
	Y.~LeCun, C.~Cortes, and C.~Burges, ``{MNIST} handwritten digit database,''
	\emph{ATT Labs}, vol.~2, 2010.
	
	\bibitem{cifar.2009}
	A.~Krizhevsky, ``Learning multiple layers of features from tiny images,'' Tech.
	Rep., 2009.
	
\end{thebibliography}
\end{document}